\documentclass[lettersize,journal]{IEEEtran}
\usepackage{amsmath,amsfonts}
\usepackage{algorithm}
\usepackage{algpseudocode}
\usepackage{hyperref}
\hypersetup{hidelinks}
\usepackage{array}
\usepackage{courier}
\usepackage{subfigure}
\usepackage{textcomp,soul,color,xcolor}
\usepackage{stfloats}
\usepackage{url}
\usepackage{verbatim}
\usepackage{graphicx}
\usepackage{booktabs}
\usepackage{multirow}
\usepackage{algpseudocode}
\usepackage{cite}
\usepackage{amsthm}
\usepackage{bm}
\theoremstyle{definition}

\theoremstyle{definition}
\newtheorem{remark}{Remark}
\theoremstyle{definition}
\newtheorem{assumption}{Assumption}
\theoremstyle{definition}
\newtheorem{lemma}{Theorem}

\hypersetup{hidelinks,
	colorlinks=true,
	allcolors=black,
	pdfstartview=Fit,
	breaklinks=true}
\newcommand{\MYhref}[3][blue]{\href{#2}{\color{#1}{#3}}}%

\title{Improved Consensus ADMM for Cooperative Motion Planning of Large-Scale Connected Autonomous Vehicles with Limited Communication}

\author{Haichao Liu, Zhenmin Huang, Zicheng Zhu, Yulin Li, Shaojie Shen, and Jun Ma

\thanks{Haichao Liu, Zhenmin Huang, Yulin Li, and Jun Ma are with the Robotics and Autonomous Systems Thrust and the Department of Electronic and Computer Engineering, The Hong Kong University of Science and Technology, China (e-mail: hliu369@connect.ust.hk; zhuangdf@connect.ust.hk; yline@connect.ust.hk; jun.ma@ust.hk).}
\thanks{Zicheng Zhu is with the Department of Electrical and Computer Engineering, National University of Singapore, Singapore (e-mail: zhuzicheng@u.nus.edu)}
\thanks{Shaojie Shen is with the Department of Electronic and Computer Engineering, The Hong Kong University of Science and Technology, China (e-mail: eeshaojie@ust.hk).}\thanks{This work has been submitted to the IEEE for possible publication. Copyright may be transferred without notice, after which this version may no longer be accessible.}
}

\begin{document}
\maketitle
\begin{abstract}

This paper investigates a cooperative motion planning problem for large-scale connected autonomous vehicles (CAVs) under limited communications, which addresses the challenges of high communication and computing resource requirements. 
Our proposed methodology incorporates a parallel optimization algorithm with improved consensus ADMM considering a more realistic locally connected topology network, and time complexity of $\mathcal{O}(N)$ is achieved by exploiting the sparsity in the dual update process.
To further enhance the computational efficiency, we employ a lightweight evolution strategy for the dynamic connectivity graph of CAVs, and each sub-problem split from the consensus ADMM only requires managing a small group of CAVs.
The proposed method implemented with the receding horizon scheme is validated thoroughly, and comparisons with existing numerical solvers and approaches demonstrate the efficiency of our proposed algorithm. Also, simulations on large-scale cooperative driving tasks involving 80 vehicles are performed in the high-fidelity  CARLA simulator, which highlights the remarkable computational efficiency, scalability, and effectiveness of our proposed development.
Demonstration videos are available at \MYhref[black]{https://henryhcliu.github.io/icadmm\_cmp\_carla}{https://henryhcliu.github.io/icadmm\_cmp\_carla}. 

\end{abstract}

\section{Introduction}
Recent developments in wireless communications and intelligent transportation systems, particularly the integration of autonomous driving with vehicle-to-everything (V2X) technologies, have provided valuable insights in addressing urban transportation challenges~\cite{chang2023bev,wegener2021longitudinal}.
Additionally, advanced planning and control methodologies for multi-agent systems have shown significant potential in the field of autonomous vehicles~\cite{Toumieh2022Decentralized,Xiao2022Distributed}.
As a result, connected autonomous vehicles (CAVs) have emerged as a promising technological solution for reducing traffic congestion, preventing accidents, and enhancing both driving safety and efficiency.
Significantly, achieving cooperative driving is a multi-faced endeavor that involves various essential aspects, including environmental perception~\cite{wen2022deep}, global path generation~\cite{zhang2023robot}, motion planning~\cite{Duan2023Cooperative}, trajectory tracking control~\cite{hu2019MME}, etc.
Within the framework of cooperative motion planning, CAVs collaboratively make joint decisions and subsequently generate cooperative trajectories aimed at optimizing overall system performance and ensuring safe and efficient maneuvers.
However, the incorporation of various key factors, such as inter-vehicle collision avoidance, vehicle model, and various physical constraints, renders cooperative motion planning a challenging task. 
This is especially evident when dealing with a larger number of CAVs or when operating in complex traffic environment.
To address the cooperative motion planning problem for CAVs, a series of representative works have been presented, drawing inspiration from advancements in both learning-based and optimization-based methodologies~\cite{claussmann2020review}.

\begin{figure}[t]
    \centering
    \includegraphics[width=\linewidth]{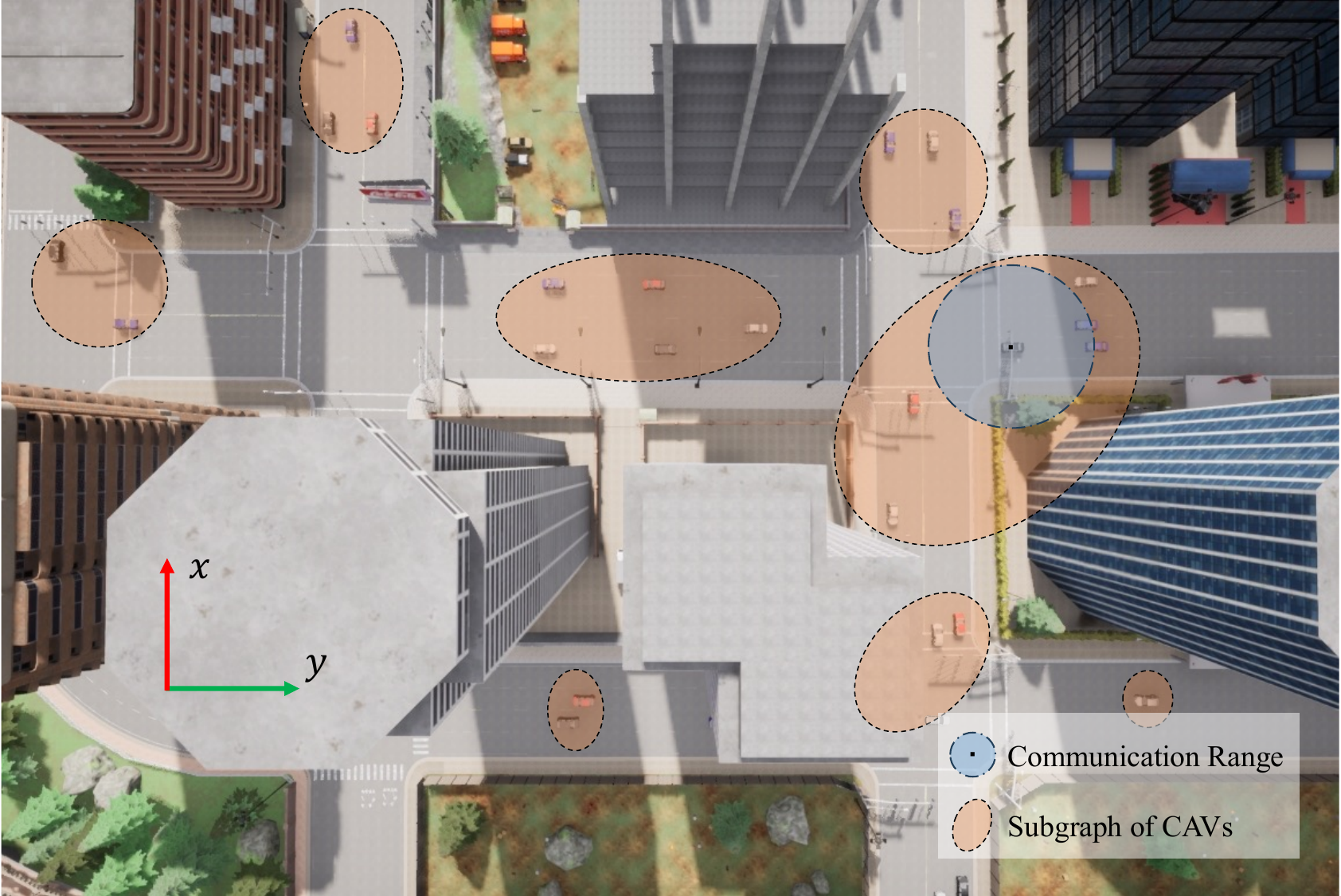}
    \caption{Demonstration of the proposed collaborative motion planning strategy for large-scale CAVs in urban driving scenarios. The problem is formulated, where one optimal control problem is constructed within each subgraph (orange region) and solved by the improved consensus ADMM algorithm. The subgraphs are generated by the proposed graph evolution algorithm, in which the edges of the dynamic connectivity graph are created within each blue region centered by each CAV within a specified subgraph.}
    \label{fig:limitedTeleRange}
\end{figure}

The use of learning-based algorithms, represented by reinforcement learning (RL), has become prevalent in tackling motion planning problems~\cite{teng2023motion,duan23relaxed}. In the context of multi-agent path finding, PRIMAL serves as a decentralized framework that plans efficient single-agent paths by imitating a centralized expert through imitation learning, thereby facilitating implicit coordination during online path planning~\cite{sartoretti2019primal}. Despite its scalability to larger teams, PRIMAL suffers from performance deterioration in structured and densely occupied environments that require substantial efforts in coordination of agents. 
To address this limitation, drawing inspiration from slime mold, RL-APCP$^3$ presents an effective solution to the deterioration problem in dense maps using bionic SARSA algorithm~\cite{liu2021solving}. Specifically for CAVs, a learning-based iterative optimization algorithm was used to explore the collision-free cooperative motion planning problem for CAVs at unsignalized intersections~\cite{klimke2022cooperative}. In addition, a novel multi-agent behavioral planning scheme for CAVs was investigated by exploiting RL and graph neural networks at urban intersections, which improves the vehicle throughput significantly~\cite{wang2023coordination}. However, the learning-based methods typically require substantial real-world data and may be limited by their interpretability, restricting their generalization to a wider range of traffic scenarios.

Alternatively, optimization-based approaches commonly adopt an optimal control problem (OCP) formulation for motion planning tasks. This approach offers several benefits such as providing a precise mathematical representation, ensuring interpretability, and proving optimality~\cite{li202optimization}. When dealing with cooperative motion planning problems involving CAVs with nonlinear vehicle models, these OCPs can be typically solved using well-established nonlinear programming solvers such as interior point optimizer (IPOPT) and sequential quadratic programming (SQP).
Additionally, the iterative linear quadratic regulator (iLQR) leverages the benefits of differential dynamic programming (DDP) for addressing nonlinear optimization problems, while retaining only the first-order term of dynamics through Gauss-Newton approximation to enhance the computational efficiency~\cite{lee2022gpu}.
While the traditional iLQR algorithm is effective and efficient in handling system dynamic constraints, this method cannot directly address various inequality constraints entailed by collision avoidance requirements and other physical limitations.
To overcome this limitation, many representative works have been presented to cope with inequality constraints within the DDP/iLQR framework, including control-limited DDP~\cite{mastalli2022feasibility,aoyama2021constrained} and constrained iLQR~\cite{chen2019autonomous,Ma2022Alternating,ma2023local}.
To further reduce the computational burden, especially for large-scale systems, the alternating direction method of multipliers (ADMM) was developed by decomposing the original optimization problem into several sub-problems~\cite{boyd2011distributed}.
With its parallel and distributed characteristics, the ADMM can be well-suited to address the cooperative motion planning problem for CAVs~\cite{zhang2021semi}.
Leveraging the dual consensus ADMM~\cite{banjac2019decentralized}, a fully parallel optimization framework was established for cooperative trajectory planning of CAVs, which effectively distributes the computational load evenly among all participants, allowing for real-time performance~\cite{huang2023decentralized}.

However, the aforementioned works rely on a strong assumption of 
fully connected topology network between CAVs, which is not applicable in real-world driving conditions under V2X communication, due to high delays and low transmission reliability associated with long-range and large-scale V2X communication~\cite{ye2023spatial}. To address the communication delays, the QuAsyADMM algorithm that incorporates a finite time quantized averaging approach was proposed, and it achieves the performance comparable to that of algorithms under the perfect communication~\cite{rikos2023asynchronous}. On the other hand, decentralized methods have also been explored, utilizing ADMM applied to the dual of the resource allocation problem, where each agent exchanges information solely with its neighbors~\cite{banjac2019decentralized,shorinwa2020scalable}. The convergence of this method has been proven, and its effectiveness has been illustrated through numerical examples.
Building upon the aforementioned algorithm, the ND-DDP algorithm was presented, where an efficient three-level architecture was proposed employing ADMM for consensus and DDP for scalability~\cite{saravanos2023distributed}. With this approach, only local communication is required. However, these solutions, which focus on commuting with neighbors, exhibit long computational time when the number of agents increases. Also, these approaches only consider a graph that does not evolve with time, where the assigned neighbors of each agent are assumed to be fixed. Yet in the real-world urban driving scenarios, the topology network between vehicles is spatiotemporally dynamic, with each node experiencing constantly fast changing neighborhoods throughout the driving process~\cite{sheng2022graph}. Nonetheless, maintaining a dynamic connectivity graph within the whole planning horizon is challenging considering the uncertainty in the future trajectories of all vehicles. 
Fortunately, receding horizon frameworks formulated with the model predictive control (MPC), 
offer the potential for 
more robust driving performance
~\cite{brudigam2021stochastic,liu2023integrated}. Nonetheless, solving a centralized MPC problem involving all the agents is rather time-consuming~\cite{qu2023rl}. Alternatively, the distributed MPC has also been widely deployed~\cite{luis2020online,Qin2023Asynchronous}, in which each agent only considers its own objective thereby obviously reducing the scale of optimization. In~\cite{cheng2021admm}, a partially parallel computation framework is proposed to solve the multi-agent MPC problem based on ADMM.
However, even if all agents can be deployed in a distributed
manner, the collision avoidance requirements (which naturally lead to the coupling effects among agents) result in a dramatic increase in the number of constraints to the optimization problem.
Therefore, all these aforementioned methods suffer from significant computational burdens, rendering them difficult to deploy in large-scale CAVs.
Apparently, to establish an efficient cooperative driving framework for an excessive number of CAVs, necessary efforts are required to decompose the overall OCP into a series of fully parallel and solvable sub-problems,
such that the problem can be solved efficiently.

Driven by the aforementioned challenges, this paper introduces an improved consensus ADMM for cooperative motion planning of large-scale CAVs with locally connected topology network, together with a graph evolution strategy to limit the scale of each OCP corresponding to each subset of the CAVs. The overview of the proposed methodology is demonstrated in Fig.~\ref{fig:limitedTeleRange}. 
Particularly, to achieve high computational efficiency, we transform and decouple the inter-vehicle constraints and leverage parallel solving capabilities of the ADMM. 
Moreover, we reduce the complexity of the dual update process by capturing the sparsity of the problem and deriving a simplified equivalent form. 
Furthermore, we improve its real-time performance by limiting the scale of the optimization problem by integrating the dynamic graph evolution algorithm. 

The key contributions of this work are as follows:

\begin{itemize}

\item An improved consensus ADMM algorithm is introduced to solve the cooperative motion planning problem with a locally connected topology network, such that the problem is decomposed into a series of sub-problems that can be solved in a parallel manner. This algorithm achieves a complexity of $\mathcal{O}(N)$ by exploiting the sparsity in the dual update process.

\item A graph evolution strategy for a dynamic connectivity graph of CAVs is presented to characterize the interaction patterns among CAVs, and each sub-problem attempts to manage a small group of CAVs. In this sense, the computational efficiency can be further enhanced.

\item A closed-loop framework for cooperative motion planning is then proposed, which integrates the graph evolution in a receding horizon fashion. This enhances the ability of consensus ADMM to cope with large-scale CAVs and provides a more robust driving performance.

\item Real-time performance and superiority of our proposed method is demonstrated through numerical simulations. The large-scale cooperative motion planning simulations involving 80 CAVs are executed in a high-fidelity urban map in CARLA, and the results showcase the  effectiveness and robustness of our proposed methodology.

\end{itemize}

The rest of this paper is organized as follows: Section II provides the definitions of notations. In Section III, we present the CAVs' connectivity pattern using graph theory and formulate the cooperative motion planning problem. Section IV reformulates the problem and an improved consensus ADMM is proposed. In Section V, we present a closed-loop receding horizon-based cooperative motion planning strategy based on our graph evolution algorithm. In Section VI, we demonstrate the performance of the proposed methodology through a series of simulations. Section VII presents the conclusion.
\section{Notations}\label{sec:notations}
We use the lowercase letter and uppercase letter to represent a vector $\bm x \in \mathbb{R}^n$ and a matrix $\bm X \in \mathbb{R}^{n\times m}$, respectively. 
Also, $\text{diag}\{a_1, a_2, \cdots, a_n\}$ represents a diagonal matrix with $a_i, i \in \{1,2,\cdots, n\}$ as diagonal entries. Similarly, $\text{blocdiag}\{\bm A_1, \bm A_2, ..., \bm A_n\}$ is used to represent a block diagonal matrix with block diagonal entries $\bm A_1, \bm A_2, ..., \bm A_n$.
The sets of non-negative integers and real numbers are represented as $\mathbb Z_+$ and $\mathbb R_+$, respectively. 
A logical/Boolean matrix is denoted as $\mathbb{B}$. $[\, \cdot\, ]^i$ represents the variable of the $i$th CAV, while $[\, \cdot\, ]_\tau$ indicates the variable at time step $\tau$. For simplicity, a matrix concatenated in row $[\bm s^1 ,\bm  s^2 , \cdots,\bm  s^n]$ is denoted as $[\bm s^i], i = 1,2,\cdots, n$. On the contrary, we denote the matrix with entries $\bm s^i$ concatenated by column as $[\bm s^1;\bm s^2;...;\bm s^n]$. The squared weighted $L_2$-norm $\bm x^T \bm M \bm x$ is simplified as $\| \bm x \|^2_{\bm M}$. We denote an \textit{indicator function} w.r.t. a set $\mathbb{\mathcal{X}}\in \mathbb{R}^n$ as $\mathcal{I}_\mathcal{X}(\bm x)$, which takes value 0 if $\bm x\in \mathbb R^n$ belongs to $\mathcal{X}$ and $+\infty$ otherwise. Element-wise maximization and minimization operations for the vectors $\bm a$ and $\bm b$ are denoted as $\max\{\bm a,\bm b\}$ and $\min\{\bm a,\bm b\}$, respectively.

\section{Problem Formulation}
\subsection{Representation of CAVs with Graph Theory}\label{sec:graphTheory}

The relationship of CAVs can be represented using graph theory, where nodes represent individual vehicles and edges represent communication links between them. This representation allows for a comprehensive analysis of the communication relationships among CAVs. 
Let $\mathcal G = (\mathcal V, \mathcal E)$ be an undirected graph, where $\mathcal V$ represents the set of vehicles and $\mathcal E$ represents the set of communication links between nodes. Each vehicle $n^i \in \mathcal V$ is associated with a state vector $\bm {z}^i$ and a control input vector $\bm u^i$. It is pertinent to note that according to the relative locations of the CAVs, the subgraphs of $\mathcal G$ are noted as $\mathcal H$, evolved between the planning horizons of cooperative driving. Particularly, even in the same subgraph, the vehicle $n^i$ only exchanges information with its surrounding vehicles within a certain communication range. 
For clarification, we denote the number of CAVs in the whole traffic system as $M$. Accordingly, the distance matrix $\bm D = [a_{ij}] \in \mathbb{R}^{M\times M}$ with $i,j\in \{1,2,\cdots,M\}$ of the whole connectivity graph $\mathcal G$ can be formulated. The number of CAVs in each subgraph $\mathcal{H}_k = (\mathcal{N}_k, \mathcal{E}_{k,h})$ is denoted as $N_k$. Note that we omit the subscript $k$ (the index of a subgraph) of the above notations by default for simplicity.
As shown in Fig.~\ref{fig:limitedTeleRange}, for each pair of vehicles, an edge $(n^i, n^j) = d^{i,j} \in \mathcal{E}_h\subseteq \mathcal{E}$ is established if the distance between the pair of vehicles $d^{i,j}$ is smaller than the communication range $r^i_\text{tele}\in \mathbb{R}_+$. In addition, we define the cardinality of a node $n^i$ in $\mathcal{H}$ as $|{n^i}| = \text{deg}(n^i)$, describing the number of indices of the connected non-self nodes in 
\begin{equation}\label{eq:neighbor_node_set}
    \mathcal{N}^i = \{j\in\mathcal{N}\, | \,(i,j)\in \mathcal{E}_h\}.
\end{equation}
The above representation offers clear and logical insights into the connectivity and communication patterns among CAVs, thereby aiding the practical implementation of decentralized cooperative motion planning algorithms. 
\subsection{Vehicle Model}\label{subsec:VehicleDynamicsModel}
For all CAVs, the non-linear bicycle kinematics model in \cite{tassa2014control} is used. The state vector of a vehicle is $\bm z=[x,y,\theta,v]^T$, where $x$ and $y$ are the positions of the rear axle of the vehicle from $x-$axis and $y-$axis in the global Cartesian map, $\theta$ 
and $v$ are the heading angle (from the positive direction of the $x$-axis) and velocity of the vehicle, respectively. The control input vector is $\bm u = [a, \delta]^T$, where $a$ and $\delta$ are the acceleration and steering angle of the front wheel, respectively. 
Then, the vehicle model can be expressed as:
\begin{equation}
\label{eq:bicycle_dynamics}
\begin{split}
    \bm z_{\tau+1} = f(\bm z_\tau,\bm u_\tau)
    =\left[\begin{array}{c}
x_{\tau}+f_i\left(v_\tau, \delta_\tau\right) \cos \theta_\tau \\
y_{\tau}+f_i\left(v_\tau, \delta_\tau\right) \sin \theta_\tau\\
\theta_\tau+\arcsin \left(\frac{g\left(v_\tau, \delta_\tau\right)}{b_i}\right) \\
v_\tau+\tau_s a_\tau 
    \end{array}\right]
\end{split},
\end{equation}
where $g(v,\delta)=v_\tau \Delta T  \sin \left(\delta_\tau\right)$,  $f_i(v,\delta)$ is defined as:
\begin{equation}
    f_i(v, \delta)=b_i+ v \Delta T \cos \delta-\sqrt{b_i^2-g(v,\delta)^2},
\end{equation}
where $b_i$ is the wheelbase of the $i$th CAV and $\Delta T=0.1$~s.
In addition, we formulate two matrices to denote the state vectors and control inputs of the $i$th CAV by $$\bm Z^i=\{\bm z_0^i, \bm z_1^i, ... , \bm z_T^i\},\bm U^i=\{\bm u_0^i, \bm u_1^i, ... , \bm u_{T-1}^i\}.$$ 
\subsection{Cooperative Motion Planning with Communication Limits}
This work considers a cooperative motion planning task in one subgraph $\mathcal{H}$ where $N$ CAVs drive on the urban road with their destinations and global path. Assuming all the vehicles share the same control and velocity limits, we formulate this problem as an OCP:
\begin{equation}
\label{NMPCOptProb}
\begin{array}{ll}
\underset{{\boldsymbol{z_\tau}^i}, {\boldsymbol{u_\tau^i}}}{\min} & \sum_{i=1}^N Q_i(\bm Z^i, \bm U^i)\\
\text { s.t.} & \boldsymbol z^i_{\tau+1}=f(\boldsymbol z^i_\tau,\boldsymbol u^i_\tau),  \\
& \bm z^i_{\tau+1}\in \mathcal{S}^i_\tau,\\
& -\boldsymbol {\underline u}^i \preceq \boldsymbol {u}_\tau^i \preceq \boldsymbol {\overline u}^i,\\
& -\boldsymbol {\underline z}^i \preceq \boldsymbol z_\tau^i \preceq \boldsymbol {\overline z}^i,\\
& -\boldsymbol {\underline z}^i \preceq \boldsymbol z_T^i \preceq \boldsymbol {\overline z}^i,\\
&\forall \tau \in \mathcal{T},\ \forall i \in \mathcal{N},
\end{array}
\end{equation}
where $\mathcal{N}$ represents the set of index of CAVs in the traffic system of a subgraph $\mathcal H$. $\mathcal{T}=\{0,1,..., T-1\}$ is the temporal horizon within a cooperative motion planning episode. 
In addition, $\mathcal{S}_\tau^i$ is the collision-free region of the $i$th CAV at time step $\tau$, which will be defined in Section~\ref{subsec:multualAvoidance}. ${\bm {\underline u}}^i$ and ${\bm {\overline u}}^i$ mean the lower and upper bound of the control input for $i$th CAV, respectively. Similar meanings take effect on ${\bm {\underline z}}^i$ and ${\bm {\overline z}}^i$.

Moreover, the objective of this OCP is to follow the reference trajectories generated by an efficient sampling-based algorithm. For an individual vehicle, it can be expressed as
\begin{equation}
    Q_i(\bm Z^i, \bm U^i) = \sum _{\tau=0}^T \| \bm z^i_\tau-\bm z^i_{\text{ref},\tau}\|^2_{\bm Q}+ \sum_{\tau=0}^
{T-1}\| \bm u^i_\tau\|^2_{\bm R},
\end{equation}
where the weighted matrices are $\bm Q = \text{diag}\{q_x, q_y, q_\theta, q_v\}$ and $\bm R = \text{diag}\{q_\delta,q_a\}$. The expected outcome of the optimization problem (\ref{NMPCOptProb}) is to cooperatively follow the reference states with different weights while complying with vehicle model and staying in the safe region $\mathcal{S}^i_\tau$ to avoid collisions with each other.

\subsection{Inter-Vehicle Collision Avoidance of CAVs}\label{subsec:multualAvoidance}

\begin{figure*}[t]
    \centering
    \includegraphics[width=\linewidth]{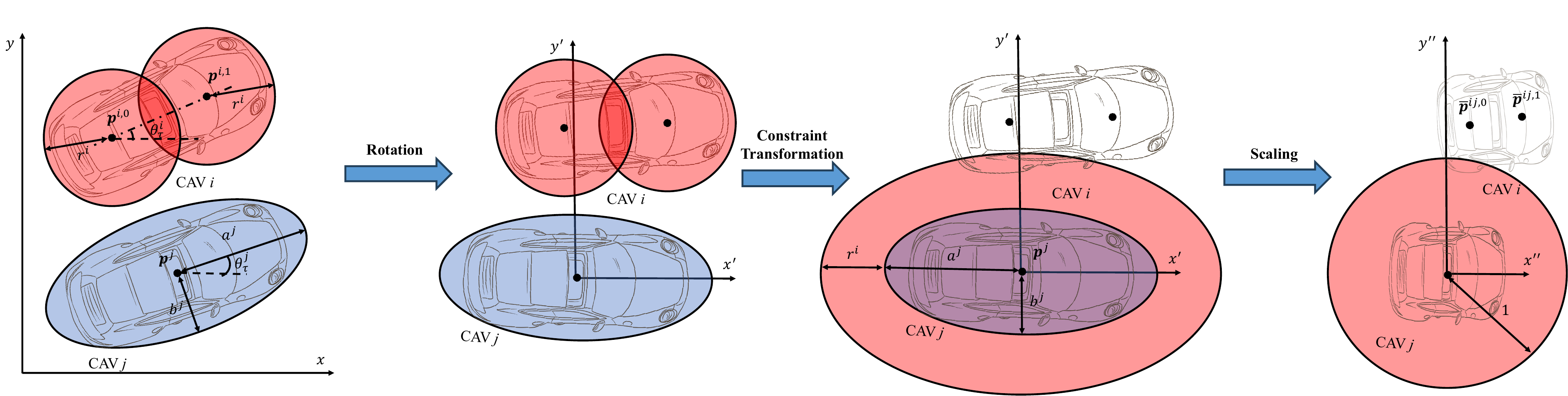}
    \caption{Geometric relationship and transformation for inter-vehicle collision avoidance, in which the two CAVs in a pair are enwrapped with a double circle and an ellipse, respectively. The radius of CAV $i$ is $r^i$, while the length of semi-major and semi-minor axes of CAV $j$ are $a^j$ and $b^j$, respectively. The rotation angle between the global coordinate system and the $a$-paralleled coordinate system is the same as the heading angle of the $\theta^j$. After the above transformation, the collision discriminant becomes judging whether a point is within a unit circle.}
    \label{graphicMutualAvoidance}
\end{figure*}
Under connected conditions, assume each CAV can receive real-time information from its neighbors in $\mathcal{N}^i$ in a certain range and collaborate to generate maneuvers that prevent collisions. 
Typically, vehicles are represented as two circles, and the criterion for collision detection is the distance between each pair of circles. For instance, to construct the safe region $\mathcal{S}^i_\tau$, if there are $|n^i|=N^i$ CAVs around $i$th CAV, there would be $2^2N^i$ pairs of constraints that need to be considered.

Alternatively, we can reduce the number of above constraints by modeling certain CAVs, such as a CAV with a lower order in a pair of CAVs, as ellipses~\cite{schwarting2017safe}, while representing the other CAV in the pair as double circles. This would decrease the constraints number by $2N^i$ for constructing the $\mathcal{S}^i_\tau$. For a CAV pair ($n^i$ and $n^j$) using the above description, we formulate the following constraints for inter-vehicle collision avoidance:
\begin{equation}
\begin{aligned}
    & \frac{(\hat{x}^j_\tau-\hat{x}^{i,c}_\tau)^2}{(a^j+r^i)^2} + \frac{(\hat{y}^j_\tau-\hat{y}^{i,c}_\tau)^2}{(b^j+r^i)^2} \ge 1,\\ 
    & \forall {c}\in \mathcal{C}, j\in \mathcal{N}^i, i\in \mathcal{N}, \tau \in \mathcal{T},
\end{aligned}
\label{ellipseEquation}
\end{equation}
where $\hat{x}$ and $\hat{y}$ means the rotated position w.r.t. the $x'-y'$ coordinate parallel to the semi-major axis $a$ of the ellipse ($a$-parallelled coordinate) of $(x,y)$, $\mathcal{C}= \{0,1\}$ denotes the front and rear circle of a CAV, $\mathcal{N}^i\in \mathbb Z_+^{\text{deg}(n^i)}$ denotes the index set of the neighbor CAVs within the communication range $r_\text{tele}^i$ of the $i$th CAV.
The discriminant (\ref{ellipseEquation}) comes into effect for the following reason. In Fig.~\ref{graphicMutualAvoidance}, we can observe that after applying rotation and translation transformations to align with the $a$-paralleled coordinate system, the origin is the center of the CAV $n^j$ and the $x'$-axis is parallel to the semi-major axis of the CAV $n^j$. Then we proceed to enlarge the ellipse by the radius of the circles $r^i$ encapsulating the other CAV. Consequently, we utilize the coordinates of the circle centers in (\ref{ellipseEquation}) to assess whether these centers fall within the enlarged ellipse. If this condition is met, it signifies a collision between the two CAVs. We will make the scaling transformation to make a more compact constraint form in Section~\ref{subsec:convexAndReform}.

For the coordinate transformation of the CAV's position from the global coordinate to the $a$-parallelled coordinate, the following equation is given:
\begin{equation}
\bm {\hat{p}}_\tau^k=\left[\begin{array}{cc}
\cos \theta_\tau^j & \sin \theta_\tau^j \\
-\sin \theta_\tau^j & \cos \theta_\tau^j
\end{array}\right] \cdot \bm p_\tau^k = \bm{R}^j_\tau \cdot \bm p_\tau^k,
\label{transformEllipseCenter}
\end{equation}
where $\bm {\hat{p}}_\tau^k = [{\hat{x}}_\tau^k,{\hat{y}}_\tau^k]^T$ is the position of the $k$th vehicle under the $a$-parallelled coordinate, while $\bm p_\tau^k$ is the position under the global $x-y$ coordinate in Fig.~\ref{graphicMutualAvoidance}. Another notation $\bm R^j_\tau$ is the transformation matrix to the $a$-paralleled coordinate. Besides, if given the center position $\bm p_\tau^{i}$ of the CAV $n^i$, the center positions of the front and rear circles are expressed as:
\begin{equation}
\bm p_\tau^{i, c}=
\bm p^i_\tau
\pm
\left[
\begin{array}{c}
\cos{\theta^i_\tau}\\
\sin{\theta^i_\tau}
\end{array}
\right] d^{i,c}.
\label{transformEllipseCircleCenter}
\end{equation}
Clearly, There are two circles $\bm p_\tau^{i,0}$ and $\bm p_\tau^{i,1}$ that need to be transformed. The centers of these circles can be transformed using (\ref{transformEllipseCenter}), which involves replacing 
$\bm p_\tau^k$ with $\bm p_\tau^{i,c}$. With this, the transformed positions of the circles, $\bm {\hat{p}}_\tau^{i,0}$ and $\bm {\hat{p}}_\tau^{i,1}$, can be obtained. Additionally, the distance of the $c$th circle to the center of the $i$th CAV is denoted by $d^{i,c}$.

\section{Cooperative Motion Planning with Limited Communication via Improved Consensus ADMM}\label{sec:pipeline}
\begin{figure}[t]
    \centering
    \includegraphics[width=1\linewidth]{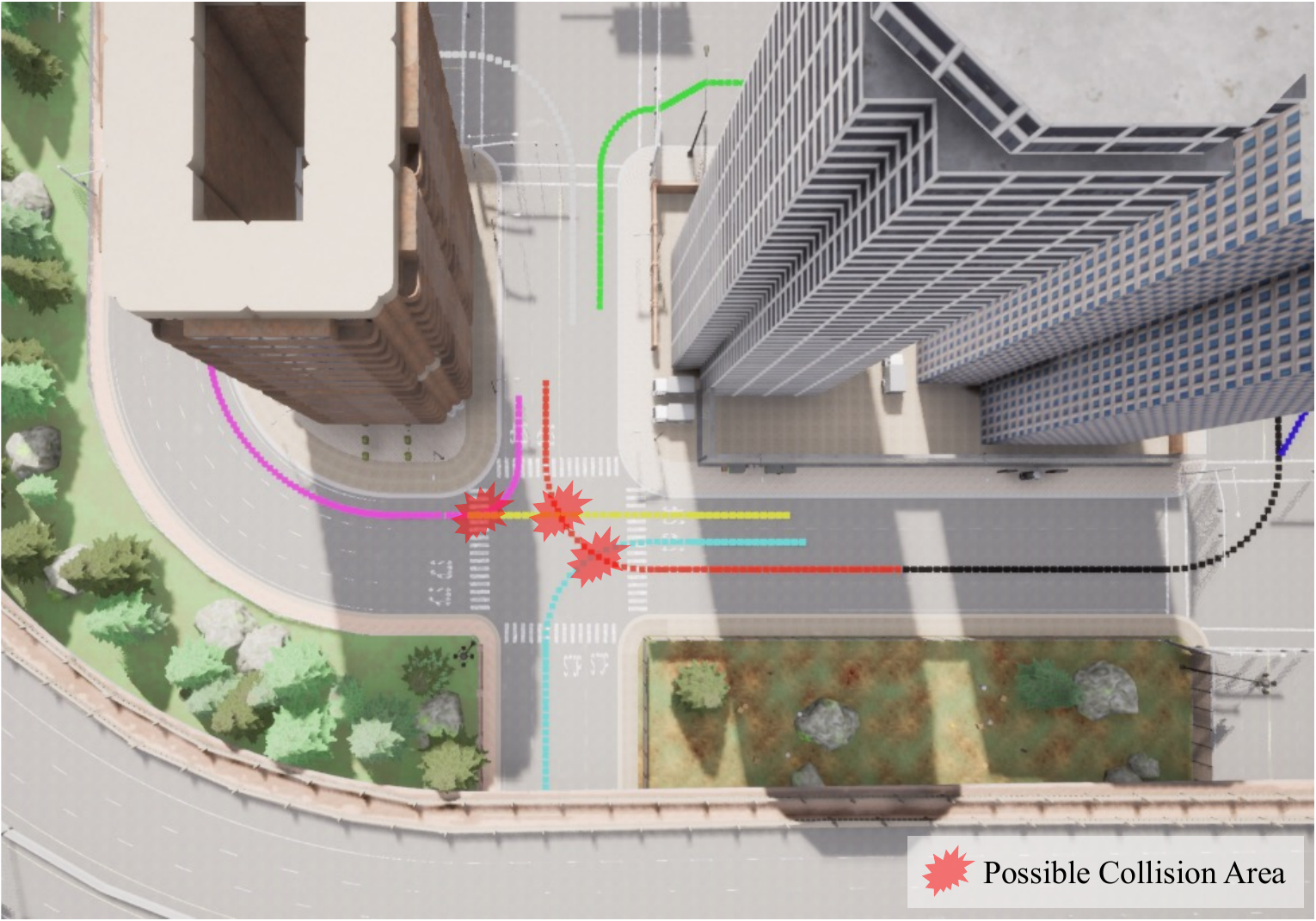}
    \caption{Smoothed guidance trajectories for cooperative motion planning. We generate these trajectories by sampling the road topology information from the OpenDRIVE map. Each vehicle's trajectory is visually represented by dots of different colors. The plotted guidance trajectories unveil potential collision conflicts that may arise within the intersection area.}
    \label{fig:initial_traj_for_efficiency_comp}
    \label{fig:initTrajGen}
\end{figure}
\subsection{Guidance Trajectory Generation and Dynamic Search}\label{guidanceTrajGen}
In the rough trajectory generation process, waypoints are sampled along the center line of lanes with equal intervals in the OpenDRIVE format map. As illustrated in Fig.~\ref{fig:initTrajGen}, use the same approach in~\cite{liu2023integrated}, for any starting point $\bm x^s = [p^{s}_x, p^{s}_y, \varphi^s]^\top$, A$^*$ search can be performed on the connectivity graph to obtain the shortest path leading to the target destination $\bm x^t = [p^{t}_x, p^{t}_y, \varphi^t]^\top$. Note that we use the Savitzky-Golay filter to improve the continuity and smoothness of the searched path, which makes the distribution of waypoints more even, especially around intersections, corners, and regions for lane changing. 
\begin{remark}
For cooperative motion planning problems, the temporal trajectories of CAVs may not always be strictly followed due to the influence of surrounding vehicles' behavior. To address this, it becomes necessary to enable CAVs to automatically follow an appropriate waypoint based on their current state. One approach to achieving this is by searching for the nearest waypoint that the vehicle should follow. 
\end{remark}
After the global reference waypoints are obtained, the local reference waypoints can be determined utilizing a Nearest Neighbor Search mechanism for updating the reference trajectory of each CAV. Concretely, at each time step, each CAV searches for the nearest waypoint using a KD-tree formulated by the waypoints in its rough reference trajectory. 

\begin{remark}
The utilization of the KD-tree data structure is shown to be highly efficient, particularly when dealing with a large number of waypoints. This guidance waypoint searching process minimally impacts the efficiency of the cooperative motion planning algorithm while offering a more robust and comfortable riding experience.
\end{remark}
\subsection{Convexification and Reformulation}\label{subsec:convexAndReform}

In this section, we assume that all the CAVs are allocated to the same subgraph $\mathcal H$. To achieve fully decentralized motion planning for the CAVs efficiently, the objective function of the OCP should be convexified, and the constraints should be linearized\cite{huang2023decentralized}. In this paper, we transform the optimization variables from $\bm z^i_\tau$ and $\bm u^i_\tau$ to perturbed forms $\bm {\Delta z}^i_{\tau}$ and $\bm {\Delta u_\tau}$.
Firstly, the vehicle model can be linearized by first-order Taylor expansion:
\begin{equation}
    \bm {\Delta z}^i_{\tau+1} = \bm A^i_{\tau}\bm {\Delta  z}^i_{\tau}+\bm B^i_{\tau}\bm {\Delta u}^i_{\tau},
\label{eq:DynamicsLinearization}
\end{equation}
where $
\bm A_\tau^i={\partial f\left(\bm z_\tau^i, \bm u_\tau^i\right)}/{\partial \bm z_\tau^i},\ \bm B_\tau^i={\partial f\left(\bm z_\tau^i,\bm u_\tau^i\right)}/{\partial \bm u_\tau^i}
$ are the partial derivatives of the vehicle model. 

Secondly, the inter-vehicle collision avoidance constraints between pairs of vehicles can be rewritten and convexified as follows. We define the following function with perturbed variable vector $\bm {\Delta\overline{p}}_\tau^k = [{\Delta\overline{x}}_\tau^k,{\Delta\overline{y}}_\tau^k]^\top$ as the discriminant for inter-vehicle collision avoidance in the $a$-paralleled coordinate:
\begin{equation}
\begin{aligned}
    & f_\mathcal{S} (\bm{z}_\tau^j, \bm{z}_\tau^{i,c}) = \bm{k}_\tau^{ij}\bm{\Delta \overline{p}}_\tau^{ij}+\|\bm{\overline{p}}_\tau^{ij} \|_2 - d_\text{safe},\\
    & \forall {c}\in \mathcal{C}, i\in \mathcal{N}^j, j\in \mathcal{N}^i, \tau \in \mathcal{T}^+,
\end{aligned}
\end{equation}
where $\bm k^{ij}_\tau = \frac{\bm{\overline{p}}_\tau^{ij}}{\|\bm{\overline{p}}_\tau^{ij}\|}$ is a unit vector for projective transformation of the perturbed displacement to the direction from the $i$th CAV to the $j$ th one, $d_\text{safe}\ge 1$ is the obstacle avoidance margin, $\mathcal{T}^+ = \{0,1,...,T\}$ , $\bm{\overline{p}}_\tau^{ij}$ is the relevant position of the pair of CAVs, and it can be calculated by:
\begin{equation}\label{eq:transform_p}
    \bm{\overline{p}}_\tau^{ij} = \bm S^{ij} \bm R^j_\tau (\bm p_\tau^{i,c} - \bm p_\tau^{j}),
\end{equation}
where $\bm S^{ij} = \text{diag}\{\frac{1}{a^j+r^i}, \frac{1}{b^j+r^i}\}$. Given the state vectors $\bm z^j_\tau$ and $\bm z^i_\tau$ as the working points, the perturbed relevant position $\bm{\Delta \overline{p}}_\tau^{ij}$ can be calculated by:
\begin{equation}\label{eq:perturbedRelevantPosition}
    \bm{\Delta \overline{p}}_\tau^{ij} = \frac{\partial \bm{\overline{p}}_\tau^{ij}}{\partial \bm z_\tau^j} \bm {\Delta z}_\tau^j + \frac{\partial \bm{\overline{p}}_\tau^{ij}}{\partial \bm z_\tau^i} \bm {\Delta z}_\tau^i.
\end{equation}
Plugging (\ref{transformEllipseCenter}), (\ref{transformEllipseCircleCenter}), and (\ref{eq:transform_p}) to (\ref{eq:perturbedRelevantPosition}), the partial derivative of the CAV $n^j$ at time step $\tau$ is
\begin{equation*}
\begin{aligned}
& \frac{\partial \bm{\bar{p}}_\tau^{i j}}{\partial \bm z_\tau^j}=\bm S^{ij} \frac{\partial \bm{R}_\tau^j \bm p_\tau^{i c}}{\partial \bm z_\tau^j}+\bm S^{ij} \frac{\partial \bm R_\tau^j \bm p_\tau^j}{\partial \bm z_\tau^j} \\
& =\bm S^{ij}\left(\left[\begin{array}{cccc}
0 & 0 & -\sin \theta_\tau^j x_\tau^{i,}+\cos \theta_\tau^j y_\tau^{i, c} & 0 \\
0 & 0 & -\cos \theta_\tau^j x_\tau^{i, c}-\sin \theta_\tau^j y_\tau^{i, c} & 0
\end{array}\right]\right. \\
& -\left.\left[\begin{array}{cccc}
\cos \theta_\tau^j & \sin \theta_\tau^j & -\sin \theta_\tau^j x_\tau^j+\cos \theta_\tau^j y_\tau^j & 0\\
-\sin \theta_t^j & \cos \theta_\tau^j & -\cos \theta_\tau^j x_\tau^j-\sin \theta_\tau^j y_\tau^j & 0
\end{array}\right]\right),
\end{aligned}
\end{equation*}
while the partial derivative of the CAV $n^i$ at time step $\tau$ is
\begin{equation*}
\begin{aligned}
\frac{\partial \bm {\bar{p}}_\tau^{i j}}{\partial \bm z_\tau^i} & =\bm S^{ij} \bm R_\tau^j \frac{\partial \bm p_\tau^{i, c}}{\partial \bm z_\tau^i} \\
& =\bm S^{ij} \bm R_\tau^j\left[\begin{array}{cccc}
1 & 0 & \mp d^{i \cdot c} \sin \theta_\tau^i & 0 \\
0 & 1 & \pm d^{i \cdot c} \cos \theta_\tau^i & 0
\end{array}\right],
\end{aligned}
\end{equation*}
where for the front and rear circles of the $i$th CAV, the third column elements in the derivative matrices have opposite signs.
The vehicle is in the safe set $\mathcal{S}^i_\tau$ if ${f}_\mathcal{S} (\bm {\hat{p}}_\tau^j, 
    \bm {\hat{p}}_\tau^{i,c})\ge 0$. 
Therefore, with the new notation of 
\begin{equation}\label{eq:JacobianForEllip}
    \bm J^j_{ic\tau} = \bm k^{ij}_\tau\bm S^{ij} \left(\frac{\partial \bm{R}_\tau^j \bm p_\tau^{i c}}{\partial \bm z_\tau^j}+\frac{\partial \bm R_\tau^j \bm p_\tau^j}{\partial \bm z_\tau^j}\right),
\end{equation}
\begin{equation}\label{eq:JacobianForCircle}
    \bm J^{i,c}_{j\tau} = \bm k^{ij}_\tau\bm S^{ij} \bm R_\tau^j \frac{\partial \bm p_\tau^{i, c}}{\partial \bm z_\tau^i},
\end{equation}
the discriminant of the safety indicator is approximated as:
\begin{equation}
\|\bm S^{ij} \bm R^j_\tau (\bm p_\tau^{i,c} - \bm p_\tau^{j})\|_2+\bm J^j_{ic\tau} \bm \Delta z_\tau^j + \bm J^{i,c}_{j\tau} \bm \Delta z_\tau^i -d_\text{safe} \ge 0,
\label{eq:CollisionAvoidanceLinearization}
\end{equation}
where $j\in \mathcal{N}^i$, $i\in \mathcal{N}^j$, $c\in \mathcal{C}$, and $\tau\in \mathcal{T}^+$. 

Third, the vehicle state and control input constraints w.r.t the perturbed variables $\bm{\Delta z}$ and $\bm{\Delta u}$ can be written as:
\begin{equation}
    \begin{aligned}
        &\bm{\underline{z}}^i-\bm z_\tau^i \leq \bm{\Delta z}_\tau^i \leq \bm{\overline{z}}^i-\bm z_\tau^i,\\
        &\bm{\underline{u}}^i-\bm u_\tau^i \leq \bm{\Delta u}_\tau^i \leq \bm{\overline{u}}^i-\bm u_\tau^i.
    \end{aligned}
\label{eq:stateControlLimitation}
\end{equation}
All the constraints of the original cooperative problem (\ref{NMPCOptProb}) are linearized and transformed by (\ref{eq:DynamicsLinearization}), (\ref{eq:CollisionAvoidanceLinearization}), and (\ref{eq:stateControlLimitation}).

Lastly, the quadratic objective function in (\ref{NMPCOptProb}) w.r.t. the perturbed independent variables $\bm {\Delta z}^i_\tau$ can be derived by the second-order Taylor Expansion, which can be described as:
\begin{equation*}
\begin{aligned}
\bar{Q}^i_\tau(\bm{\Delta z}^i_\tau, \bm{\Delta u}^i_\tau)&\approx Q^i_\tau (\bm{z}^i_\tau,\bm{u}^i_\tau) + \bm{\Delta u}^{i\top}_\tau Q^i_{\tau,u} + \bm{\Delta z}^{i\top}_\tau Q^i_{\tau,z}\\
&+\frac{1}{2}\bm{\Delta u}^{i\top}_\tau Q^i_{\tau,uu}\bm{\Delta u}^{i}_\tau +\frac{1}{2}\bm{\Delta z}^{i\top}_\tau Q^i_{\tau,zz}\bm{\Delta z}^{i}_\tau,
\end{aligned}
\end{equation*}
where $\tau\in \mathcal{T}$, ${Q}^i_\tau( \bm{\Delta z}^i_\tau, \bm {\Delta u}^i_\tau)=Q^i_\tau(\bm{z}^i_\tau+\bm{\Delta z}^i_\tau, \bm u^i_\tau+\bm {\Delta u}^i_\tau)$. In addition, $\bm Q^i_{\tau,{u}}$ and $\bm Q^i_{\tau,{z}}$ are the partial derivatives of $Q^i_\tau$ w.r.t. $\bm{u}^i_\tau$ and $\bm{z}^i_\tau$, while $\bm Q^i_{\tau,{u}{u}}$ and $\bm Q^i_{\tau,{z}{z}}$ are the second-order partial derivatives of $Q^i_\tau$ w.r.t. $\bm{u}^i_\tau$ and $\bm{z}^i_\tau$. Note that the state vector and the control input vector are decoupled, hence the second-order mixed derivatives $\bm Q^i_{\tau,uz} =\bm  Q^i_{\tau,zu} =\bm  0$. With a similar definition of $\bar{Q}_\tau^i$, the terminal cost is 
$$
\bar{Q}^i_T(\bm {\Delta z}^i_T)\approx Q^i_T(\bm z^i_T)+\bm {\Delta z}^{i\top}_T \bm Q^i_{T,z}+\frac{1}{2}\bm{\Delta z}^{i\top}_T Q^i_{T,zz}\bm {\Delta z}^{i}_T.
$$

The following part reformulates the OCP as a consensus optimization problem with a standard consensus ADMM form~\cite{banjac2019decentralized}. To facilitate the description of the independent variables of the cooperative problem, define $\bm{\Delta Z}^i = [\bm{\Delta z}^{i\top}_0, \bm{\Delta u}^{i\top}_0,...,\bm{\Delta z}^{i\top}_T]^\top\in\mathbb{R}^{6T+4}$ as the concatenated vector of all the state vectors and control inputs corresponding to the $i$th CAV. Besides, define $\bm L^i_1=[Q^i_{0,z},Q^i_{0,u},\cdots,Q^i_{T,z}]\in\mathbb{R}^{6T+4}$ as the column vector containing all the first-order Jacobians, and $\bm L^i_2=\textup{blocdiag}\{Q^{i}_{0,zz},Q^{i}_{0,uu},\cdots,Q^{i}_{T,zz}\}\in\mathbb{R}^{(6T+4)\times (6T+4)}$ as the block diagonal matrix of all second-order Hessians. Hence, the perturbed overall host cost is
\begin{equation}
\begin{aligned}
&\bar{F}^i(\bm{\Delta Z}^i)= \bm{\Delta Z}^{i\top}\bm L^i_1+\frac{1}{2}\bm{\Delta Z}^{i\top}\bm L^i_2\bm{\Delta Z}^i.
\end{aligned}
\end{equation}
With a similar expression, the vehicle model (\ref{eq:DynamicsLinearization}) can be described by
\begin{equation}
(\bm L^i_3-\bm L^i_4) \bm{\Delta Z}^i=0,
\label{eq:indicatorLinearizedMatrix}
\end{equation}
where 
$\bm L^i_3=\text{blocdiag}\{[\bm A^i_0\ \bm B^i_0],[\bm A^i_1\ \bm B^i_1],\cdots,[\bm A^i_{T-2}\ \bm B^i_{T-2}],$
$[\bm A^i_{T-1}\ \bm B^i_{T-1}\ \bm{0}^{4\times 4}]\}\in \mathbb R^{(4T)\times(6T+4)}$
and
$L^i_4=\text{blocdiag}\{[\bm{0}^{4\times 4}\ \bm{0}^{4\times 2}\ \bm I^4],[\bm{0}^{4\times 2}\ I^4],\cdots,[\bm{0}^{4\times 2}\ I^4]\}\in \mathbb R^{(4T)\times(6T+4)}$.
Then, we define the solution set of the linearized vehicle model as 
\begin{equation}
    \mathcal{S}^D = \{\bm{\Delta Z}^i|(\bm L^i_3-\bm L^i_4) \bm{\Delta Z}^i=0\}, 
\end{equation}
which is nonempty, closed, and convex.
In conclusion, when combining the indicator function related to $\mathcal{S}^D$ with the cost associated with the host, the resulting expression can be written as:
\begin{equation}
\label{eq:FDef}
F^i(\bm {\Delta Z}^i)= \bar{F}^i(\bm {\Delta Z}^i) + \mathcal I_{\mathcal S^D}(\bm{\Delta Z}^i).
\end{equation}

Afterward, we formulate the inter-vehicle collision avoidance constraints and vehicle physical constraints.
We combine the inequality constraints into another indicator function $\mathcal{I}_{\mathcal{K}}$ corresponding to the feasible set of (\ref{eq:CollisionAvoidanceLinearization}) and (\ref{eq:stateControlLimitation}). 
For inter-vehicle collision avoidance constraints, the Jacobian matrices in (\ref{eq:JacobianForEllip}) and (\ref{eq:JacobianForCircle}) should be concatenated in the following order. For $i$th CAV, its concatenated Jacobian at time step $\tau$ is
\begin{equation}\label{eq:JacobianOrder}
\bm J_\tau^i=\left[\bm a_{j k}^{\top}\right]_{1 \leq j<k \leq N}^{\top}\in \mathbb{R}^{N(N-1)\times 4},
\end{equation}
where according to (\ref{eq:JacobianForEllip}) and (\ref{eq:JacobianForCircle}), $\bm a_{j k}$ can be expressed as
\begin{equation}
\bm a_{j k}=\left\{\begin{array}{ll}
[\bm J_{k 0 \tau}^i ; \bm J_{k i \tau}^i]& j=i \\
{[\bm J_{j \tau}^{i, 0} ; \bm J_{j \tau}^{i, 1}]}& k=i \\
\bm 0^{2\times 4} & \text{otherwise.}
\end{array}\right.
\end{equation}
With the expression of the Jacobian $\bm J^i_\tau$, we define $\bm {\hat{J}}^i = \text{blocdiag}\{[\bm J^i_0\ \bm 0^{N(N-1)\times 2}],\cdots,[\bm J^i_{T-1}\ \bm 0^{N(N-1)\times 2}],\bm J^i_T\}$. 
Further, with a slight abuse of notation, we define the norm vector $\bm k_\tau$ constructed by $\|\bm S^{ij} \bm R^j_\tau (\bm p_\tau^{i,c} - \bm p_\tau^{j})\|_2$ at time step $\tau$ with a similar modality of (\ref{eq:JacobianOrder}), which is
\begin{equation}
\bm k_\tau=\left[\left\|\bm S^{i j} \bm R_\tau^j\left(\bm p_\tau^{i, c}-\bm p_\tau^j\right)\right\|_2\right]_{1 \leq j<i \leq N}\in \mathbb{R}^{N\times (N-1)},
\end{equation}
where $c \in \mathcal C$. Concatenating $\bm k_\tau$ at all time steps, we define $\bm {\hat{k}} = [\bm k_\tau]_{\tau \in \mathcal{T}^+}\in \mathbb{R}^{N(N-1)(T+1)}$.
Then, we rewrite the collision avoidance discriminant (\ref{eq:CollisionAvoidanceLinearization}) of all the CAVs through the planning horizon as:
\begin{equation}
\bm {\hat{k}}+\sum_{i=1}^N \bm {\hat{J}}^i \bm {\Delta Z}^i-\bm d_{\text {safe }}^{N(N-1)(T+1)} \succeq 0.
\end{equation}

For the box constraints in (\ref{eq:stateControlLimitation}), we define a new matrix $\bm {\Tilde{O}^i}$ step by step. We define a new variable $\bm{\Delta e}^i=\left[\bm {\Delta \hat{e}}^{i\top},\Delta v^i_T\right]^\top\in \mathbb R^{3T+1}$ with the notation of $\bm {\Delta \hat{e}}^i = \left[[\Delta v^i_\tau,\Delta \delta^i_\tau,\Delta a^i_\tau]^\top\right]_{i\in\mathcal{T}}$, where $\bm{\Delta e}^i$ denotes the variables with box constraints in the dynamic system (\ref{eq:bicycle_dynamics}). Then $\bm{\Delta e}^i = \bm O^i\bm{\Delta Z^i}$, where
$$
\bm O^i = \text{blocdiag}\{\underbrace{[\bm 0^{3\times 3},\bm I^3],\cdots}_T,[\bm 0^{1\times 3}, 1]\}\in \mathbb B^{(3T+1)\times (6T+4)}.
$$
To facilitate the form $\bm{\Delta e} = \sum_{i=1}^N \bm{\hat{O}}^i\bm{\Delta Z^i}$ required in the dual consensus ADMM algorithm~\cite{grontas2022distributed}, we further define 
\begin{equation}
    \bm{\hat{O}}^i=[\underbrace{\bm{0}^{(3T+1) \times (6T+4)}, \cdots}_{i-1}, \bm O^i, \underbrace{\bm{0}^{(3T+1) \times (6T+4)}, \cdots}_{N-i}],
\label{eq:SparseTildeOi}
\end{equation}
where $\bm{\hat{O}}^i \in \mathbb{B}^{(3T+1)N \times (6T+4)}$. Hence, the discriminant (\ref{eq:stateControlLimitation}) for all the CAVs can be aggregated and rewritten as
\begin{equation}
    \bm{\underline E} - \bm E_\tau \leq \sum_{i=1}^N \bm{\hat{O}}^i\bm{\Delta Z^i} \leq \bm{\overline E} - \bm E_\tau,
\end{equation}
where $\bm E_\tau = [\bm e^i]_{ i\in \mathcal{N}}$, is the concatenated state and control input vectors of all the vehicles at time step $\tau$. $\bm{\underline E}$ and $\bm{\overline{E}}$ denote the lower and upper bounds, respectively.
With the notation $\bm k = [-\bm{\hat{k}}+\bm d_{\text {safe }}^{N(N-1)(T+1)},\bm 0^{(3T+1)N}]\in \mathbb R^{N(N-1)(T+1)+(3T+1)N}$, we formulate the second element of the unconstrained problem by the following two functions 
$$G_1(\bm M_{g1}) = \mathcal{I}_{\mathcal R_+}(\bm M_{g1}-\bm{{k}}),\ G_2(\bm M_{g2}) = \mathcal{I}_{\mathcal S^b}(\bm M_{g2}),$$where the subscripts $\mathcal{R}_+$ denotes a positive semidefinite cone with a $N(N-1)(T+1)$ dimension and $\mathcal{S}^b = \{\bm s|\bm{\underline E} - \bm E_\tau \preceq \bm s \preceq \bm{\overline E} - \bm E_\tau\}$ denotes a continuous convex set with a $(3T+1)N$ dimension, respectively. The independent variables are
$$
\bm M_{g1} = \sum_{i=1}^N \left(\bm {\hat{J}}^i\bm{\Delta Z^i}\right),\ 
\bm M_{g2} = \sum_{i=1}^N \bm{\hat{O}}^i\bm{\Delta Z^i}.
$$
Concatenating $\bm{\hat{J}}^i$ with the upper and lower bound constraints in (\ref{eq:SparseTildeOi}) as $\bm J^i = [\bm{\hat{J}}^i; \bm{\hat{O}}^i]$, we define the concatenated function $$G(\bm M) = G(\bm M_{[1]};\bm M_{[1]}) = \mathcal{I}_{\mathcal{K}}\left(\sum_{i=1}^N \bm J^i\bm {\Delta Z}^i-\bm k\right),$$
where $\bm M = \sum_{i=1}^N \bm J^i\bm {\Delta Z}^i$ and $\mathcal{K} = \{(\bm a;\bm b)|\bm a\in \mathcal{R}_+, \bm b\in \mathcal{S}^b\}$ is the Cartesian product of the two sets $\mathcal{R}_+$ and $\mathcal{S}^b$.
Note that $G(\cdot)$, the sum of two convex functions $G_1(\cdot)$ and $G_2(\cdot)$, is still convex. At this stage, by including $G(\cdot)$ in the overall cost incurred by $F(\cdot)$, with a new notation of $\bm k^i = \bm k/N$, the unconstrained convex optimization problem is given by:
\begin{equation}
\min_{\bm {\Delta Z}^i} \sum_{i=1}^N F^i(\bm {\Delta Z}^i)+\mathcal{I}_{\mathcal{K}}\left(\sum_{i=1}^N \left(\bm J^i\bm {\Delta Z}^i-\bm k^i\right)\right).
\label{eq:unconstrainedOCP}
\end{equation}

\subsection{Improved 
Consensus ADMM for Cooperative Motion Planning with Limited Communication Range}\label{sec:basicDcooperative motion planning}\label{sec:BoostingEfficiency}
In this part, we leverage a parallel optimization framework based on dual consensus ADMM and improve its computational efficiency. For completeness, we restate the preliminaries of dual consensus ADMM~\cite{banjac2019decentralized}, as also summarized in Algorithm~1. We add the relax variable $\bm h$ and adapt the dual formulation of the problem (\ref{eq:unconstrainedOCP}) as:
\begin{equation}
    \begin{aligned}
        \min_{\bm {\Delta Z}^1,...,\bm{\Delta Z}^N} &\sum_{i=1}^N F^i(\bm {\Delta Z}^i)+\mathcal{I}_{\mathcal{K}}\left(\bm h\right)\\
        \text{s.t. } & \sum_{i=1}^N \left(\bm J^i\bm {\Delta Z}^i-\bm k^i\right)=\bm h.
    \end{aligned}
\label{ConstrainedNewProb}
\end{equation}
In this problem, Slater's condition holds under affine constraints, which ensures the strong duality property. The Lagrangian of (\ref{ConstrainedNewProb}) is given by:
\begin{equation}
\begin{aligned}
    \mathcal{L}(\bm{\Delta Z},h,\bm y)&=\sum^N_{i=1} F^i(\bm{\Delta Z}^i)+\mathcal{I}_{\mathcal{K}}(\bm h)\\
    &+\bm y^\top\left(\sum^N_{i=1}\left(\bm J^i \bm{\Delta Z}^i-\bm k^i\right)-\bm h\right),
\end{aligned}
\end{equation}
where $\bm y\in\mathbb{R}^{N(N-1)(T+1)+(3T+1)N}$ represents the dual variable relating to the inter-vehicle collision avoidance constraints and box constraints from the vehicle physics. The Lagrange dual function can be expressed as:
\begin{equation}
\begin{aligned}
    h(\bm y)&=\inf_{\bm{\Delta Z},h}\mathcal{L}(\bm {\Delta Z},\bm h,\bm y)\\
    &=-\sum^N_{i=1}\left(F^{i*}\left(-\bm J^{i\top}\bm y\right)+\mathcal{I}_{\mathcal{K}^{\circ}}(\bm y)+\bm y^\top \bm k^i\right).
\end{aligned}
\end{equation}
Note that $F^{i*}$ represents the conjugates of $F^i$ and $\mathcal{K}^\circ = \left\{\bm y\in \mathbb R^{N(N-1)(T+1)+(3T+1)N}|\text{sup}_{\bm h\in \mathcal{K}}\ \bm y^\top \bm h\leq0\right\}$ is the \textit{polar cone} of the convex cone $\mathcal{K}$. The problem of the Lagrange dual problem, i.e., $$\max_{\bm y}h(\bm y)=\max_{\bm y}\,-\sum^N_{i=1}\left(F^{i*}(-\bm J^{i\top}\bm y)+\mathcal{I}_{\mathcal{K}^{\circ}}(\bm y)+\bm y^\top \bm k^i\right),$$ can be treated as a consensus optimization problem that is decomposed. In this problem, all vehicles collaborate to collectively reach a decision on a shared optimization variable, denoted as $\bm y$. Consequently, the consensus ADMM Algorithm~\ref{alg:alg1} can be applied.

\begin{algorithm}[t]
\caption{Dual Consensus ADMM~\cite{grontas2022distributed}}\label{alg:alg1}
\begin{algorithmic}[1]
\State \textbf{choose} $\sigma,\rho >0$
\State \textbf{initialize} for all $i \in \mathcal{N}$: $p^{i,0}=y^{i,0}=x^{i,0}=s^{i,0}=0$
\State \textbf{repeat}: for all $i \in \mathcal{N}$
\State \hspace{0.3cm} Broadcast $y^{i,k}$ to the vehicles in $\mathcal{N}^i$
\State \hspace{0.3cm} $p^{i,k+1}=p^{i,k}+\rho\sum_{j\in \mathcal{N}^i}(y^{i,k}-y^{j,k})$
\State \hspace{0.3cm} $s^{i,k+1}=s^{i,k}+\sigma(y^{i,k}-x^{i,k})$
\State \hspace{0.3cm} $r^{i, k+1} = \sigma x^{i, k}+\rho \sum_{j\in \mathcal{N}^i}\left(y^{i, k}+y^{j, k}\right)$
\Statex \hspace{1.3cm} $-(k^i+p^{i, k+1}+s^{i, k+1})$
\State \hspace{0.3cm} $z^{i, k+1} = \arg \min \left\{F^i\left(z^i\right)\right.\left.+\gamma\left\|J^i z^i+r^{i, k+1}\right\|^2\right\}$
\State \hspace{0.3cm} $y^{i, k+1} = 2\gamma\left(J^i z^{i, k+1}+r^{i, k+1}\right)$
\State \hspace{0.3cm} $x^{i, k+1} = \Pi_{\mathcal{K}^\circ}\left(\frac{1}{\sigma}s^{i, k+1}+y^{i, k+1}\right)$
\State \hspace{0.3cm} $k=k+1$
\State \textbf{until} termination criterion is satisfied
\end{algorithmic}
\end{algorithm}

Leveraging the methodological results in~\cite{huang2023decentralized}, we are able to get the analytical expressions of the column vectors $\bm p,\bm  s,\bm  r,\bm  y, \bm x\in \mathbb{R}^{N(N-1)(T+1)+(3T+1)N}$ shown in Algorithm~\ref{alg:alg1}. The improved consensus ADMM for cooperative motion planning problem contains the following two steps:
\subsubsection{Dual Update}
The decentralized local problems corresponding to each CAV achieves consensus by dual update process. In Algorithm \ref{alg:alg1}, we define $\gamma = \frac{1}{2(\sigma+2\rho |n^i|)}$.
Steps 5-7 and 9-10 are used to update the dual variables. Note that in Step 5 and Step 7, the $i$th CAV only processes the Jacobian Matrices related to its own neighbors in $\mathcal{N}^i$. All steps are straightforward, except for Step~10. Based on the Moreau decomposition~\cite{bauschke10convex}, the projection operator can be expressed as $\Pi_{\mathcal{K}^\circ}(\bm a) = \bm a - \Pi_{\mathcal{K}}(\bm a).$
Therefore, Step~10 can be rewritten as
\begin{equation}
\bm x^{i, k+1}=\frac{1}{\sigma} \bm s^{i, k+1}+\bm y^{i, k+1}-\bm \lambda^{i, k+1},
\end{equation}
where $\bm \lambda^{i,k+1}$ is an intermediate variable calculated by:
\begin{equation}\label{eq:maxMin}
\begin{aligned}
\bm \lambda^{i, k+1} & =\Pi_{\mathcal{K}}\left(\frac{1}{\sigma}\bm s^{i, k+1}+\bm y^{i, k+1}\right)\\
& = \underset{\bm q \in \mathcal{K}}{\operatorname{argmin}}\left\|\bm y^{i, k+1}+\frac{1}{\sigma} \bm s^{i, k+1}-\bm q\right\| \\
& =\max \left\{\bm \Lambda, \min \left\{\bm \Gamma, \bm y^{i, k+1}+\frac{1}{\sigma} \bm s^{i, k+1}\right\}\right\},
\end{aligned}
\end{equation}
where $k$ means the number of the ADMM iteration.
Based on the definition of the Cartesian product $\mathcal{K}$, we denote the element-wise lower bound and upper bound as $\bm \Lambda = \left[\bm 0^{N(N-1)(T+1)},\bm{\underline E}\right]+\bm \epsilon$ and $\bm \Gamma = \left[\bm \infty^{N(N-1)(T+1)},\bm{\overline E}\right]-\bm \epsilon$, respectively. Here, $\bm \epsilon = [\epsilon]^{N(N-1)(T+1)+(3T+1)N}$ is a vector filled with a small non-negative number $\epsilon$ to push $\bm x^{i,k+1}$ away from the boundary into the feasible region, ensuring strict satisfaction of constraints in the original problem. 
\begin{remark}
    Due to the inherent decoupled nature of this optimization problem with respect to each element of $\bm x$, the solution can be easily obtained through element-wise maximization and minimization as described in (\ref{eq:maxMin}).
\end{remark}


\subsubsection{Primal Update}
Note that Step 8 of Algorithm~\ref{alg:alg1} executes the primal update to explore the optimal states and control inputs for the $i$th CAV. The primal variable $z^{i,k+1}=\bm {\Delta Z}^{i,k+1}$ denotes the augmented state of the $i$th CAV at the $(k+1)$th ADMM iteration. It can be updated by solving a standard LQR problem:
\begin{equation}
\begin{aligned}
    \min_{\bm {\Delta Z^i}}\, &\bm{\Delta Z}^{i\top}\bm L^i_1+\frac{1}{2}\bm{\Delta Z}^{i\top}\bm L^i_2\bm{\Delta Z}^i + 2\bm r^{i,k+1\top}\bm J^i\bm{\Delta Z}^{i}\\
    &+\frac{1}{\sigma+2\rho |n^i|}\bm{\Delta Z}^{i\top}\bm J^{i\top}\bm J^i\bm{\Delta Z}^{i}  \\
    \text{s.t. }\, &\left(\bm L^i_3-\bm L^i_4\right) \bm{\Delta Z}^i=0.
\end{aligned}
\label{eq:quadOptProb}
\end{equation}
It is worth noting that all terms of the objective function in (\ref{eq:quadOptProb}) are quadratic, and the constraint for the vehicle model is linear. In this sense, the problem can be effectively solved using dynamic programming. Besides, the problem solely involves state variables and control inputs that are specific to a single CAV. Consequently, the problem described by (\ref{eq:quadOptProb}) for different CAVs can be concurrently solved in a parallel manner. 
\begin{assumption}\label{fixed_iteration_assumption}
    The maximum iterations for both the outer and inner loops in Algorithm~\ref{alg:alg1} are fixed and the planning horizon in each iteration is constant.
\end{assumption}
\begin{remark}
    Although the LQR problem can be efficiently solved using the algebraic Riccati equation, it still requires dealing with the high-dimensional Jacobian matrix $\bm J$ to obtain its inverse matrix. This process can be time-consuming, especially for large-scale cooperative motion planning problems. With Assumption~\ref{fixed_iteration_assumption}, if all the CAVs in $\mathcal{H}$ are fully connected, the complexity of the Algorithm~\ref{alg:alg1} is $\mathcal{O}(N^3)$, due to Step 5 and Step 7.
\end{remark}

It is evident that the dual variables ${\bm p, \bm s, \bm r,\bm  y,\bm x}\in \mathbb R^{N(N-1)(T+1)+(3T+1)N}$ are constructed corresponding to the inter-vehicle collision avoidance constraints $[\, \cdot \, ]_{[1]}$ with a dimension of $N(N - 1)(T+1)$ and the box constraints  $[\, \cdot \, ]_{[2]}$ with a dimension of $(3T+1)N$. Therefore, separate treatments are logical for these variable sets.

With the limited communication condition, $[\, \cdot \, ]^i_{[1]}$ is calculated with the information from $\mathcal{N}^i$. In other words, the communication range limits the number of the non-zero Jacobian matrices for the inter-vehicle collision avoidance from $N(N-1)(T+1)$ to $N^i(N-1)(T+1)$, while the rest jacobian matrices are filled with $\bm 0$. The number of terms to be summed in Step 5 and Step 7 are converted from $N$ to $N^i$ as well. Assuming $N^i$ is a limited number, we derive the conclusion that the complexity of the calculation of $[\, \cdot \, ]_{[1]}$ is decreased from $\mathcal{O}(N^3)$ to $\mathcal{O}(N)$ compared to the dual consensus ADMM algorithm assuming the fully connected topology network between CAVs~\cite{huang2023decentralized}. 

On the other hand, $[\, \cdot \, ]_{[2]}$ has redundant variables associated with other CAVs. We define $[\, \cdot \,]_{[2,i]}\in \mathbb R^{3T+1}$ as the $\left((i-1)(3T+1)+1\right)$ to $i(3T+1)$ rows of $[\, \cdot \, ]_{[2]}$ pertaining to the box constraints. 
Based on the sparsity of $\bm{\hat O}^i$, we can reduce computational complexity by eliminating redundant entities except $\bm{O}^i$ for $[\, \cdot \,]^i_{[2,i]}$. Assume $i\in \mathcal{N}^v$, as outlined in Algorithm~\ref{alg:alg1}, the dual variables for the $v$th CAV w.r.t. the $i$th CAV are iteratively updated using the following equations:
\begin{subequations}
\begin{align}
\label{subeq:update_p}
\bm{p}_{[2, i]}^{v, k+1} & =\bm{p}_{[2, i]}^{v, k}+\rho \sum_{j\in \mathcal{N}^v}\left(\bm{y}_{[2, i]}^{v, k}-\bm{y}_{[2, i]}^{j, k}\right), \\
\bm{s}_{[2, i]}^{v, k+1} & =\bm{s}_{[2, i]}^{v, k}+\sigma\left(\bm{y}_{[2, i]}^{v, k}-\bm{x}_{[2, i]}^{v, k}\right), \\
\label{subeq:update_r}
\bm{r}_{[2, i]}^{v, k+1} & =\sigma \bm{x}_{[2, i]}^{v, k}+\rho \sum_{j\in \mathcal{N}^v}\left(\bm{y}_{[2, i]}^{v, k}+\bm{y}_{[2, i]}^{j, k}\right)\notag\\
&-\left(\bm k^v_{[2,i]}+\bm{p}_{[2, i]}^{v, k+1}+\bm{s}_{[2, i]}^{v, k+1}\right), \\
\label{subeq:update_y}
\bm{y}_{[2, i]}^{v, k+1} & =\left\{\begin{array}{ll}
2\gamma\left(\bm O^i \bm{\Delta Z}^{i, k+1}+\bm{r}_{[2, i]}^{i, k+1}\right)& v=i \\
2\gamma \bm{r}_{[2, i]}^{v, k+1}& v\neq i,
\end{array}\right.\\
\bm{x}_{[2, i]}^{v, k+1} &= \Pi_{\mathcal{S}^{u\circ}}\left(\frac{1}{\sigma}\bm s^{v, k+1}+\bm y^{v, k+1}\right),
\end{align}
\label{eq:optimizedTargetV}
\end{subequations}

\begin{lemma}\label{theorem1}
    Given that the dual variables of the CAVs are initialized by same values, for the elements of dual variables with $\forall j\neq i, \forall v\neq i$, $\bm \alpha^{j,k}_{[2,i]} = \bm \alpha^{v,k}_{[2,i]} =\bm \alpha^{k}_{[2,i]}, \bm \alpha\in \{\bm p, \bm s, \bm r,\bm  y,\bm x\}, \forall k\in \mathcal{K}_\text{iter}$ holds.
\end{lemma}
\begin{proof}
    Denote the statement of $$P(m):\,\bm \alpha^{j,m}_{[2,i]} = \bm \alpha^{v,m}_{[2,i]}, \bm \alpha\in \{\bm p, \bm s, \bm r,\bm  y,\bm x\},$$and the proof is given by induction on $m$.

    \noindent\textit{Base case:}
    With the same initialization of the dual variables, we have $\bm \alpha^{j,0}_{[2,i]} = \bm \alpha^{v,0}_{[2,i]}=\bm \alpha^{0}_{[2,i]}, \bm \alpha\in \{\bm p, \bm s, \bm r,\bm  y,\bm x\}$, so $P(0)$ is clearly true.

    \noindent\textit{Induction step:}
    Assume the induction hypothesis that $\forall j\neq i, [\, \cdot \, ]^{j,k}_{2,i} = [\, \cdot \, ]^{k}_{2,i}$ holds, meaning $P(k)$ is true. According to (\ref{eq:optimizedTargetV}), for the second term of the equations~(\ref{subeq:update_p}) and (\ref{subeq:update_r}), because $\bm{y}_{[2, i]}^{v, k}=\bm{y}_{[2, i]}^{j, k}=\bm{y}_{[2, i]}^{k}, j\neq i, k\neq i$, we have:
    \begin{equation}\label{eq:sum1}
        \sum_{j\in \mathcal{N}^v}\left(\bm{y}_{[2, i]}^{v, k}-\bm{y}_{[2, i]}^{j, k}\right) = \bm{y}_{[2, i]}^k-\bm{y}_{[2, i]}^{i, k},
    \end{equation}
    \begin{equation}\label{eq:sum2}
        \sum_{j\in \mathcal{N}^v}\left(\bm{y}_{[2, i]}^{v, k}+\bm{y}_{[2, i]}^{j, k}\right) = (2 |n^v|-1) \bm{y}_{[2, i]}^k+\rho \bm{y}_{[2, i]}^{i, k}.
    \end{equation}
Plug (\ref{eq:sum1}) and (\ref{eq:sum2}) to (\ref{eq:optimizedTargetV}) under the case of $v\neq i$, the following equation holds:
\begin{subequations}
\label{eq:finalImprovedBoxConstraintsOfSVinEVOCP}
\begin{align}
\label{eq:sub1}
\bm{p}_{[2, i]}^{v, k+1} & =\bm{p}_{[2, i]}^k+\rho\left(\bm{y}_{[2, i]}^k-\bm{y}_{[2, i]}^{i, k}\right) = \bm{p}_{[2, i]}^{k+1}, \\
\label{eq:sub2}
\bm{s}_{[2, i]}^{v, k+1} & =\bm{s}_{[2,1]}^k+\sigma\left(\bm{y}_{[2, i]}^k-\bm{x}_{[2, i]}^k\right)=\bm{s}_{[2, i]}^{k+1}, \\
\bm{r}_{[2, i]}^{v, k+1} & =\sigma \bm{x}_{[2, i]}^k+\rho(2 |n^v|-1) \bm{y}_{[2, i]}^k+\rho \bm{y}_{[2, i]}^{i, k}\notag  \\
&\label{eq:sub3}
-(\bm{k}_{[2, i]}^{k+1}+\bm{p}_{[2, i]}^{k+1}+\bm{s}_{[2, i]}^{k+1})=\bm{r}_{[2, i]}^{k+1},\\
\label{eq:sub4}
\bm{y}_{[2, i]}^{v, k+1} & =2\gamma \bm{r}_{[2, i]}^{k+1}=\bm{y}_{[2, i]}^{k+1}, \\
\label{eq:sub5}
\bm{x}_{[2, i]}^{v, k+1} & =\Pi_{\mathcal{S}^{b\circ}}\left(\frac{1}{\sigma}\left(\bm{s}_{[2, i]}^{k+1}+\bm{y}_{[2, i]}^{k+1}\right)\right)=\bm{x}_{[2, i]}^{k+1},
\end{align}
\end{subequations}
which means the statement $P(k +1)$ also holds, establishing the induction step from $P(k)$ to $P(k+1)$. Since both the \textit{base case} and the \textit{induction step} have been proved as true, $P(m)$ holds for every $m\in \mathcal{K}_\text{iter}$.
This completes the proof of Theorem~\ref{theorem1}.
\end{proof}

Therefore, the dual update with $v\neq i$ can be completed by (\ref{eq:finalImprovedBoxConstraintsOfSVinEVOCP}). Besides, according to (\ref{eq:optimizedTargetV}) under the case of $v= i$, with Theorem~\ref{theorem1}, the dual update of $\bm \alpha_{2,i}^{i,k+1}$ can be simplified as:
\begin{subequations}
\label{eq:finalOptimizedTargetV}
\begin{align}
\label{subeq:imp1}
\bm{p}_{[2, i]}^{i, k+1} & =\bm{p}_{[2, i]}^{i, k}+\rho(N-1)\left(\bm{y}_{[2, i]}^{i, k}-\bm{y}_{[2, i]}^k\right), \\ 
\label{subeq:imp2}
\bm{s}_{[2, i]}^{i, k+1} & =\bm{s}_{[2, i]}^{i, k}+\sigma\left(\bm{y}_{[2, i]}^{i, k}-\bm{x}_{[2, i]}^{i, k}\right), \\ 
\bm{r}_{[2, i]}^{i, k+1} & =\sigma \bm{x}_{[2, i]}^{i, k}+\rho(N-1)\left(\bm{y}_{[2, i]}^{i, k}+\bm{y}_{[2, i]}^k\right)\notag\\
\label{subeq:imp3}
&-(\bm{k}_{[2, i]}^{i, k+1}+\bm{p}_{[2, i]}^{i, k+1}+\bm{s}_{[2, i]}^{i, k+1}), \\
\label{subeq:imp4}
\bm{y}_{[2, i]}^{i, k+1} & =2\gamma\left(\bm O^i \bm{\Delta Z}^{i, k+1}+\bm{r}_{[2, i]}^{i, k+1}\right), \\
\label{subeq:imp5}
\bm{x}_{[2, i]}^{i, k+1} & =\Pi_{\mathcal{S}^{b\circ}}\left(\frac{1}{\sigma}\left(\bm{s}_{[2, i]}^{i, k+1}+\bm{y}_{[2, i]}^{i, k+1}\right)\right) .
\end{align}
\end{subequations} 
\begin{algorithm}[t]
\caption{Improved Consensus ADMM for Cooperative Motion Planning Within One Subgraph $\mathcal H(\mathcal{N},\mathcal{E}_h)$}\label{alg:alg2}
\begin{algorithmic}[1]
\State \textbf{initialize} $\{x^i_\tau,u^i_\tau\}^T_{\tau=0},\{p^{i,0},y^{i,0},z^{i,0},s^{i,0}\}, \forall i\in \mathcal{N}$
\State \textbf{choose} $\sigma,\rho >0$
\State \textbf{repeat}:
\State \hspace{0.3cm} Send $\{z^i_\tau\}_{\tau=1}^T$, receive $\{z^j_\tau\}_{\tau=1}^T$ from $j\in\mathcal{N}^i$
\State \hspace{0.3cm} Compute $k^i$, $J^i$, $\{A^i_\tau\}^{T-1}_{\tau=0}$, $\{B^i_\tau\}^{T-1}_{\tau=0}$
\State \hspace{0.3cm} \textbf{reset} $p^{i,0}=s^{i,0}=0, y^{i,0}=y^\text{last},x^{i,0}=x^\text{last}$
\State \hspace{0.3cm} \textbf{reset} $y^{i,0}_{[2]} = x^{i,0}_{[2]}=0$
\State \hspace{0.3cm} \textbf{repeat}: for all $i \in \mathcal{N}$
\State \hspace{0.6cm} Send $y^{i,k}$, receive $y^{j,k}$ from $j\in\mathcal{N}^i$
\State \hspace{0.6cm} Steps 4-7 of Algorithm~\ref{alg:alg1} for $[\, \cdot \,]^{i,k}_{[1]}$
\State \hspace{0.6cm} Perform (\ref{subeq:imp1}-\ref{subeq:imp3}) for $ \alpha^{i,k}_{[2,i]}, \alpha\in\{ p, s,  r\}$
\State \hspace{0.6cm} Perform (\ref{eq:sub1}-\ref{eq:sub3}) for $ \alpha^{i,k}_{[2,j]},\alpha\in\{ p, s,  r\},j\in\mathcal{N}^i$
\State \hspace{0.6cm} Compute $z^{i, k+1}$ by solving LQR problem (\ref{eq:quadOptProb})
\State \hspace{0.6cm} $y^{i, k+1}_{[1]} = 2\gamma\left(\hat J^i z^{i, k+1}_{[1]}+r^{i, k+1}_{[1]}\right)$
\State \hspace{0.6cm} ${x}_{[1]}^{i, k+1} = \Pi_{\mathcal{R}^{\circ}_+}\left(\frac{1}{\sigma}s^{i, k+1}_{[1]}+ y^{i, k+1}_{[1]}\right)$
\State \hspace{0.6cm} Perform (\ref{subeq:imp4}-\ref{subeq:imp5}) for $ \alpha^{i,k}_{[2,i]}, \alpha\in\{ y, x\}$
\State \hspace{0.6cm} Perform (\ref{eq:sub4}-\ref{eq:sub5}) for $\alpha^{i,k}_{[2,j]}, \alpha\in\{ y,x\},j\in\mathcal{N}^i$
\State \hspace{0.6cm} $k=k+1$
\State \hspace{0.3cm} \textbf{until} number of iteration steps exceeds $k_\text{max}$
\State \hspace{0.3cm} Update $\{z^i_\tau,u^i_\tau\}^T_{\tau=0}$
\State \textbf{until} termination criterion is satisfied
\end{algorithmic}
\end{algorithm}

\begin{remark}
    Combining (\ref{eq:finalImprovedBoxConstraintsOfSVinEVOCP}) and (\ref{eq:finalOptimizedTargetV}), the decoupled ADMM iteration of variables corresponding to the box constraints of vehicle $i$ is completed. Therefore, compared to (\ref{eq:optimizedTargetV}), the computational complexity is reduced obviously from $\mathcal O(N^i)$ to $\mathcal O(1)$, which improves the computational efficiency but does not affect the optimization performance. 
\end{remark}

In summary, combining the solving process of $[\, \cdot \, ]_{[1]}$ and $[\, \cdot \, ]_{[2]}$, an improved decentralized cooperative motion planning algorithm with a computation complexity of $\mathcal{O}(N)$ is provided. As demonstrated in Algorithm~\ref{alg:alg2}, it starts by initializing the dual variables and choosing values for hyper-parameters. It then enters a loop where it sends and receives information among the CAVs in each $\mathcal{N}^i$. It computes the augmented variables based on the received neighbor information and performs computations for each time step $\tau$. Inside the loop, it also performs additional computations and updates the primal variables until a termination criterion is satisfied. The inner loop continues to iterate until a maximum number of steps is reached. Finally, it updates the state vectors and control inputs along the whole planning horizon until the termination criterion of the outer loop is satisfied.
\section{Cooperative Motion Planning Framework for Large-Scale CAVs}
This section focuses on improving the practicality of the proposed algorithm toward cooperative motion planning with large-scale CAVs under a {locally connected topology network}. We leverage graph evolution and receding horizon strategies to achieve this objective.

\subsection{Graph Evolution based on Manhattan Distance}\label{subsec:TaskAllocationNode}
\begin{algorithm}[t]
\caption{Graph Evolution for CAVs}\label{alg:NodeDistribution}
\begin{algorithmic}[1] 
\State \textbf{input}: current positions $p^i_{T_e},\, i\in \mathcal{V}$
\State \textbf{initialize}: distance matrix $D = [d^{i,j}], i,j\in \mathcal{V}$ 
\State \textbf{reset}: threshold distance $d_\text{safe}^{i,j}=\infty$
\If{$\Delta\theta < 45^\circ$}
\State $d_\text{safe}^{i,j}= T_s \cdot \max\{v_\text{ref}^i, v_\text{ref}^j\}$
\Else
\State $d_\text{safe}^{i,j}=T_s\cdot (v_\text{ref}^i+v_\text{ref}^j)$
\EndIf
\State \textbf{build}: adjacency matrix $A_\mathcal{G}$ with condition $d^{i,j}<d_\text{safe}^{i,j}$
\State \textbf{initialize}: subgraphs list, visited list
\For{$i$ in range (len($D[0]$))} 
\If{visited[$i$] is 0}
\State Create subgraph with $i$ 
\State Mark $i$ as \texttt{visited} 
\State Initialize queue with $i$ 
\While{queue} 
\State Get node and neighbors from queue 
\For{neighbor in neighbors} 
\If{visited[neighbor] is 0} 
\State Add neighbor to subgraph and queue 
\State Mark neighbor as \texttt{visited}  
\EndIf 
\EndFor 
\EndWhile 
\State Perform (\ref{eq:edgeCreation}) to create the edge set $\mathcal{E}_h$
\State Add subgraph $\mathcal{H}$ to subgraphs list
\EndIf 
\EndFor 
\end{algorithmic} 
\end{algorithm}
To enable efficient decentralized cooperative motion planning for CAVs, a graph evolution strategy is presented incorporating node distribution algorithm and edge construction rules. The node distribution mechanism is demonstrated in Algorithm~\ref{alg:NodeDistribution}. Leveraging the notations in Section \ref{sec:graphTheory}, each node $n^i$ is distributed into a subset of nodes $\mathcal{N({H})}$ from $\mathcal{V(G)}$, 
which belong to two graphs $\mathcal{H}(\mathcal{N},\mathcal{E}_h)$ and $\mathcal {G(V,E)}$, respectively.
Their relationship can be described as:
\begin{equation}
    \mathcal{H}(\mathcal{N}, \mathcal{E}_h)\subseteq \mathcal{G}(\mathcal{V}, \mathcal{E}).
\end{equation}
The node distribution algorithm aims to establish the communication and collaboration among CAVs with certain discrimination rules, referred to as the threshold distance  $d_\text{safe}^{i,j}$. To divide the graph into as many subgraphs as possible under the inter-subgraph safety guarantee for the node distribution strategy, we propose a \textit{safe distance} by the following conditional Manhattan distance:

\begin{equation}
d_{\text {safe }}^{i, j}=\left\{\begin{array}{l}
T_s \cdot \max \left\{v_{\text {ref }}^i, v_{\text {ref }}^j\right\}, \Delta\theta<\frac{\pi}{4} \\
T_s \cdot\left(v_{\text {ref }}^i+v_{\text {ref }}^j\right), \text { otherwise.}
\end{array}\right.
\end{equation}
Vehicles with a Manhattan distance $d^{i,j}=\|\bm p^i - \bm p^j\|_1$ larger than the safe distance have no potential to collide with each other, which is explained by the following statements.

If $\Delta \theta < \frac{\pi}{4}$, the two CAVs appear to be driving in the same direction. In this case, the situation of the CAVs can be treated as a \textit{Pursuit Problem}, where the relative velocity $\Delta v^{ij} = v^i-v^j<\max\{v^i, v^j\}$. Therefore, the relative displacement of the CAVs can be expressed as 
\begin{equation}
    \Delta d^{ij} =T_s\cdot  \Delta v^{ij} < T_s\cdot \max\{v^i, v^j\} .
\end{equation}
    
If $\Delta \theta \geq \frac{\pi}{4}$, the two CAVs are either \textit{driving vertically} or \textit{toward each other}. 
In the sub-case of \textit{driving vertically}, a necessary and insufficient condition for the possibility of collision is given by \textit{Condition 1}: 
\begin{itemize}
    \item $v^i_\text{ref}\times T_s \geq \Delta x^{ij}_\tau$ and $v^j_\text{ref}\times T_s \geq \Delta y^{ij}_\tau$, or
    \item $v^j_\text{ref}\times T_s \geq \Delta x^{ij}_\tau$ and $v^i_\text{ref}\times T_s \geq \Delta y^{ij}_\tau$
\end{itemize}
For \textit{Condition 1}, an equivalent necessary and insufficient condition is listed as \textit{Condition 2}: 
\begin{itemize}
    \item $v^i_\text{ref}\times T_s + v^j_\text{ref}\times T_s \geq \Delta x^{ij}_\tau+\Delta y^{ij}_\tau =\|\bm p^i - \bm p^j\|_1$
\end{itemize}
Hence, \textit{Condition 2} is also a necessary and insufficient condition for collision avoidance. 
In the sub-case of \textit{driving toward each other}, the situation of the CAVs can be treated as an Encounter Problem, where the relative velocity $\Delta v^{ij} \leq v^i_\text{ref}+v^j_\text{ref}$. Therefore, the relative displacement of the CAVs can be expressed as $\Delta d^{ij} = \Delta v^{ij}\times T_s \leq (v^i_\text{ref}+v^j_\text{ref}) \times T_s$. If $\Delta d^{ij} \leq \|(x^i_\tau,y^i_\tau) - (x^j_\tau,y^j_\tau)\|_1$, there is no possibility of encountering with a time horizon $T_s$, which means 
\begin{equation}
    \|(x^i_\tau,y^i_\tau) - (x^j_\tau,y^j_\tau)\|_1 > (v^i_\text{ref}+v^j_\text{ref}) \times T_s.
\end{equation}
At this stage, all the cases are discussed the safety guarantee of conditional Manhattan distance is proved completely.
\begin{remark}
    In urban traffic systems, roads are typically vertically staggered, and vehicles navigate based on the topology of the road network. Hence, we motivate the use of Manhattan distance to measure the safe range between CAVs. Each CAV has a target velocity, which is generally maintained for extended periods. Therefore, the product of the target velocity and the planning horizon can serve as a criterion for establishing a safety distance between any pair of CAVs.
\end{remark}
As shown in Algorithm~\ref{alg:NodeDistribution}, a distance matrix $\bm D$ is firstly initialized to store the pairwise conditional Manhattan distances between $n^i \in \mathcal{V}$. The adjacency matrix $\bm A_\mathcal{G}$ is then constructed based on an element-wise comparison of the distance between two CAVs that is smaller than the safe distance i.e., $(n^i,n^j) = d^{i,j}< d_\text{safe}^{i,j}$. Hence, $\bm A_\mathcal{G}$ represents the connectivity between CAVs.
Next, the algorithm initializes the subgraphs and a visited array to keep track of visited CAVs. It iterates through each CAV in $\mathcal{V}$ and checks if it has been visited. If not, a new subgraph $\mathcal{H}$ is created with the current CAV as the starting node. The node is marked as visited and a queue is initialized with the current CAV.
Within the ``while'' loop, the algorithm retrieves a node $n^j$ from the queue along with its neighbors. For each neighbor that has not been visited, it is added to the subgraph $\mathcal{H}$ and the queue. The neighbor is also marked as visited. This process continues until the queue is $\bm \emptyset$.
Finally, the algorithm create the edges of each subgraph and store it to the list of subgraphs $\mathcal{H}^1,\mathcal{H}^2,\cdots,\mathcal{H}^K$. This process repeats until all CAVs have been visited, resulting in a collection of subgraphs. Within one set of nodes created by Algorithm~\ref{alg:NodeDistribution}, the edges to construct the subgraph $\mathcal{H}(\mathcal{N},\mathcal{E}_h)$ between node pairs are created by:
\begin{equation}\label{eq:edgeCreation}
    \mathcal{E}_h = \left\{(n^i,n^j),j\in\mathcal{N}|d^{i,j}\leq r_\text{tele}^i \right\}.
\end{equation}
\begin{remark}
    Algorithm~\ref{alg:NodeDistribution} provides a foundation for the construction of decentralized cooperative motion planning problems by maintaining the adjacency matrix $\bm A_{\mathcal{G}}$ among CAVs. These sets of nodes $\mathcal{N}^k$ serve as the basis for CAV coordination in the subsequent stages of the motion planning process.
\end{remark}

\subsection{Large-Scale Cooperative Motion Planning with
Receding Horizon and Graph Evolution}\label{subsec:SuccCoMotionPlanningStr}
The pipeline of the proposed cooperative motion planning framework involves optimizing the trajectories of the CAVs in a planning horizon of $T_s$ and executing only the first $T_e$ steps. We set $T_s> T_e$, which can cater to uncertainties and dynamics faults during the $(T_s-T_e)$ time horizon constructing a closed-loop structure. Besides, the horizon of the whole driving process from the starting point to the destination is denoted as $T$. 
With a receding horizon manner, it is not necessary to detect vehicles that are far away, because the maximum velocity of each CAV is limited, which defines a boundary for the trip distance within the limited planning horizon. This feature benefits the graph evolution algorithm to construct more subgraphs to effectively control the scale of the OCPs corresponding to each subgraph.

\begin{algorithm}[t] \caption{Large-Scale Cooperative Motion Planning with Receding Horizon and Graph Evolution}\label{alg:alg3}
\begin{algorithmic}[1] 
\State \textbf{obtain} starting and target positions of all CAVs in $\mathcal{G}$
\State \textbf{determine} optimal route using A* algorithm 
\State \textbf{smooth} the waypoints in each route using the Savitzky-Golay filter
\While{termination criterion is not satisfied}: 
\State Execute Algorithm~\ref{alg:NodeDistribution} to distribute the CAVs to $\mathcal{H}^k$
\For{subgraph $\mathcal{H}$ in subgraphs list} 
\State \textbf{formulate} the OCP (\ref{ConstrainedNewProb}) for the CAVs in $\mathcal{H}$
\State \textbf{search} nearest waypoints for each CAV at each time step $\tau\in \{0,1,...,T_s\}$ using KD-Tree
\State Execute Algorithm \ref{alg:alg2} to obtain CAVs' trajectories
\EndFor
\State \textbf{perform} first $T_e$ steps of the CAVs in $\mathcal{G}$
\State \textbf{feedback} at time step $T_e$
\State \textbf{relay} the states of all the CAVs in $\mathcal{G}$
\EndWhile
\end{algorithmic} 
\end{algorithm}

The cooperative motion planning strategy with a receding horizon is presented in Algorithm~\ref{alg:alg3}. Initially, the starting and target positions of all CAVs are determined by users. A preliminary route with guidance waypoints is generated and smoothed based on the road topology provided by the OpenDRIVE map, using the A* algorithm and the Savitzky-Golay filter. In the ``while" loop, the graph evolution is performed using Algorithm~\ref{alg:NodeDistribution}, and OCPs are formulated within each subgraph. To facilitate cooperative maneuvers, the nearest neighbor search is performed using an efficient KD-Tree structure within the subsequent $T_s$ steps. After formulating the decentralized cooperative motion planning optimization problem and parallelly executing Algorithm 2 for each subgraph $\mathcal{H}$, the trajectories of all CAVs are obtained. The first $T_e$ trajectories of the CAVs are executed, and the state and environment feedback at time step $T_e$ is collected and relayed between the CAVs. Simultaneously, $\mathcal{G}$ is updated by Algorithm~\ref{alg:NodeDistribution}. The aforementioned steps are repeated to obtain the remaining trajectories during $T_e$ to $T$.

\section{Simulation Results}\label{sec:simulation}

\subsection{Environment Setup}\label{subsec:expSetup}

We implement the proposed efficient decentralized cooperative motion planning framework in Ubuntu 20.04 LTS with Intel(R) Xeon(R) Gold 6230 CPU @ 2.10GHz CPU and NVIDIA TITAN RTX GPU with 256 GB RAM and $24\times 4$ GB Graphic Memory. We perform the algorithms using Python 3.8 with Numba acceleration. The simulation platform is CARLA~0.9.14~\cite{dosovitskiy2017carla} with \texttt{Town05}.
The pertinent parameters of the CAVs are depicted in Table~\ref{tab:vehicleParams}. It is notable that $d^{i,1}=-0.28$~m means that the rear circle of the target vehicle is in front of the virtual center (center of the rear axle) of the vehicle model~(\ref{eq:bicycle_dynamics}).
\begin{table}[h]
  \centering
  \caption{Parameter Settings of the CAVs}
  \resizebox{1.0\linewidth}{!}{
    \begin{tabular}{cccccc}
    \toprule
    Param. & Description & Value & Param. & Description & Value \\
    \midrule
     $b$    & wheelbase & 2.4 m & $d^{i,1}$    & rear distance  & -0.28 m \\
     $l$    & vehicle length & 3.8 m & $a^j$     & semi-major  & 3 m \\
    $w$     & vehicle width  & 1.7 m & $b^j$     & semi-minor  & 1.1 m \\
    $d^{i,0}$    & front distance  & 2.68 m & $r^i$     & radius  & 2.55 m \\
    \bottomrule
    \end{tabular}%
    }
  \label{tab:vehicleParams}%
\end{table}%
In the remainder of this section, we conduct two sets of simulation experiments to evaluate the computational efficiency of the proposed optimization algorithm and the driving performance of the proposed decentralized cooperative motion planning strategy, respectively.

We use the same guidance trajectories of the CAVs in the same map for all the benchmark methods in the comparative experiment. Concretely, as shown in Table~\ref{tab:paramsForCavGeneration}, we adopt different hyper-parameters for the reference trajectories generation of different scales of CAVs. The CAVs' spawn points $\bm z^i_{\text{ref},0}$ and destinations $\bm z^i_{\text{ref},T}$ are randomly generated. Afterward, the rough guidance trajectories containing waypoints $z^i_{\text{ref},\tau}, \tau\in [1,2,...,T]$ are searched as described by Section~\ref{guidanceTrajGen}. The generated spawn points are within a range of $[d_\text{min},d_\text{max}]$ from $(-188\,\text{m},-91.5\,\text{m})$, the center position of an intersection in \texttt{Town 05}, while the destinations of CAVs are within a range of $[d_\text{des,min},d_\text{des,max}]$ from their spawn points $z^i_{\text{ref},0}$. Lastly, to simulate real urban driving scenarios, the target velocity of each CAV $v^i_{\text{ref}}$ is randomly assigned between $[v_\text{min}, v_\text{max}]$. 
\begin{table}[h]
  \centering
  \caption{Parameters for Driving Scenario Construction}
    \begin{tabular}{ccccccc}
    \toprule
    \#\,CAV  & $d_\text{min}$ & $d_\text{max}$ & $d_\text{des,min}$ & $d_\text{des,max}$ & $v_\text{min}$ & $v_\text{max}$\\
    \midrule
    8     & 7.5\,m   & 30.0\,m    & 100.0\,m   & 150.0\,m   & 10.0\,m/s     & 10.0\,m/s \\
    16    & 7.5\,m   & 70.0\,m    & 100.0\,m   & 150.0\,m   & 5.0\,m/s     & 20.0\,m/s \\
    32    & 10.0\,m    & 140.0\,m   & 120.0\,m   & 150.0\,m   & 5.0\,m/s     & 20.0\,m/s \\
    80    & 10.0\,m    & 340.0\,m   & 140.0\,m   & 170.0\,m   & 5.0\,m/s     & 20.0\,m/s \\
    \bottomrule
    \end{tabular}%
  \label{tab:paramsForCavGeneration}%
\end{table}%
\subsection{Cooperative Motion Planning Under One Subgraph}\label{subsec:computationalEfficiency}
\subsubsection{Computational Efficiency Comparison}
We compare our optimization algorithm for solving cooperative motion planning problems with existing commonly used solvers~(the IPOPT solver and the SQP solver provided by CasADi~\cite{andersson2019casadi}) and advanced algorithms~(centralized iLQR (CiLQR) with log barrier function~\cite{chen2019autonomous}, and decentralized iLQR (DiLQR) with soft collision avoidance~\cite{huang2023decentralized}). The hyper-parameters of the optimization algorithms are listed in Table~\ref{tab:optHyperParams}. 
\begin{table}[htbp]
  \centering
  \caption{Parameter Settings of the Optimization Algorithms}
    \begin{tabular}{cc|cc|cc}
    \toprule
    Parameter & Value & Parameter & Value & Parameter & Value \\
    \midrule
    $\sigma$ &   0.05    & $\tau_s$ &  0.1~s     & $\varphi_b$ & 0.6~rad \\
    $\rho$ &   0.002    & $a_\text{max}$ &   3.0~m/s$^2$   & $T_s$    &   15 \\
    $\epsilon$    &   0.1    & $a_\text{min}$ & -5.0~m/s$^2$      &$T_e$ & 10   \\
    \bottomrule
    \end{tabular}%
  \label{tab:optHyperParams}%
\end{table}%

Computation time of different optimization horizons ($T$) and CAV numbers ($N$) are compared using a static CAV graph $\mathcal{G}$ without graph evolution. As depicted in Table~\ref{tab:compTimeEfficiency_T}, we compare the computation time of various optimization algorithms with a planning horizon ($T_s$) ranging from 30 to 90 and a fixed CAV number ($N=8$). 
IPOPT and SQP, typical solving methods in CasADi, have high computation time and low efficiency for CAVs' motion planning tasks. SQP struggles with planning horizons greater than 40. CiLQR is more time-efficient due to iterative linearization but faces significantly increased computation time as the planning horizon grows. In contrast, DiLQR, which uses ADMM to decentralize the problem, maintains efficiency even with an increased planning horizon scale. Our proposed method has also high computational efficiency with a larger planning horizon.
From another perspective, as shown in Table~\ref{tab:compTimeEfficiency_N}, for most of the compared methods, a higher number of CAVs has a greater impact on the computational efficiency than a greater planning horizon. In particular, DiLQR is much less scalable than our improved consensus ADMM algorithm.

\begin{table}[htbp]
  \centering
  \caption{Computation Time with Different Optimization Horizons}
  \resizebox{1.0\linewidth}{!}{
    \begin{tabular}{cccccccr}
    \toprule
    \multirow{2}[4]{*}{Method} & \multicolumn{7}{c}{Optimization horizon with $N$ = 8} \\
\cmidrule{2-8}          & 30    & 40    & 50    & 60    & 70    & 80    & \multicolumn{1}{c}{90} \\
    \midrule   
    CiLQR~\cite{chen2019autonomous} &0.276	&0.361	&0.577	&1.314	&2.031	&1.413	&2.099\\
    IPOPT~\cite{andersson2019casadi} &4.311&	8.701&	18.602&	31.630&	95.909&---&---\\
    SQP~\cite{andersson2019casadi}   &758.680&---    &---    &---    &---    &---    &---\\
    DiLQR~\cite{huang2023decentralized} &0.398&	0.339&	0.614&	0.582&	0.713&	1.069&	1.473\\
    Ours  & 0.197&	0.377&	0.407&	0.550&	0.733&	1.035	&1.434 \\
    \bottomrule
    \end{tabular}%
    }
  \label{tab:compTimeEfficiency_T}%
\end{table}%

\begin{table}[htbp]
  \centering
  \caption{Computation Time with Different Numbers of CAVs}
  \resizebox{1.0\linewidth}{!}{
    \begin{tabular}{cccccccrr}
    \toprule
    \multirow{2}[4]{*}{Method} & \multicolumn{8}{c}{Number of CAVs with $T$ = 30} \\
\cmidrule{2-9}          & 4     & 8     & 12    & 16    & 20    & 24    & \multicolumn{1}{c}{28} & \multicolumn{1}{c}{32} \\
    \midrule
    CiLQR~\cite{chen2019autonomous} &0.157&	0.276&	1.763&	6.001&	16.901&	15.888&	19.395&	20.110\\
    IPOPT~\cite{andersson2019casadi} &0.935&	4.311&	130.961 &173.720&314.550& ---   &  ---  &---\\
    SQP~\cite{andersson2019casadi}   &124.320&758.680&---    &---    &---    &---    & ---   &---\\
    DiLQR~\cite{huang2023decentralized}&0.073&	0.398&	0.301&	0.490&	1.096&	3.343&	5.146&	5.419\\
    Ours &0.152&	0.197&	0.565&	0.792&	0.364&	1.159&	1.887&	2.522
\\
    \bottomrule
    \end{tabular}%
    }
  \label{tab:compTimeEfficiency_N}%
\end{table}%

\subsubsection{Driving Performance Analysis}
\begin{figure}[t]
    \centering
    \includegraphics[width=0.8\linewidth]{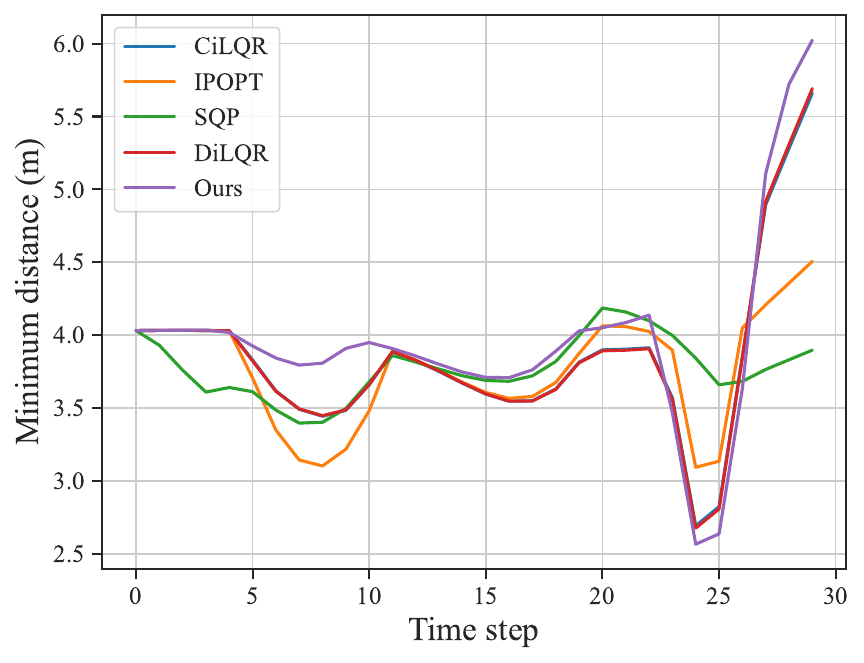}
    \caption{The minimum distance between all the CAVs at each time step. Different colors represent different methods. All the above methods performed safe cooperative motion planning with a minimum distance greater than 2.5~m.}
    \label{fig:efficiency_comp_min_dis}
\end{figure}

\begin{figure}[t]
    \centering
    \includegraphics[width=0.8\linewidth]{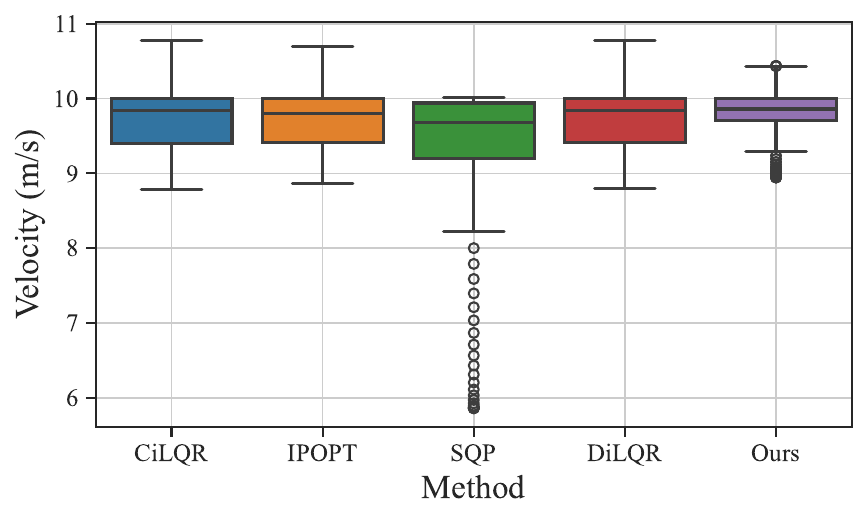}
    \caption{Velocity distribution along all the time steps with different methods. As the reference velocity is 10 m/s, a speed distribution closer to this value indicates better driving performance.}
    \label{fig:efficiency_comp_vel_distribution}
\end{figure}

As the cooperative motion planning problem (\ref{NMPCOptProb}) is non-convex, it is challenging to guarantee finding the global optimal solution using any solver. Therefore, it becomes necessary to establish evaluation criteria to assess the optimization performance. In this paper, the motion planning performance is evaluated based on two criteria: the minimum distance between any pair of CAVs and the velocity distribution throughout the entire driving process.

As depicted in Fig.~\ref{fig:efficiency_comp_min_dis}, all the methods exhibit a similar trend and successfully achieve safe driving performance in terms of the minimum distance between CAVs. However, when considering driving efficiency, notable differences arise among the methods. As demonstrated in Fig.~\ref{fig:efficiency_comp_vel_distribution}, the SQP method exhibits the largest deviation and variance from the reference velocity, with the lowest mean velocity among all the alternatives. 
Compared with other algorithms, our proposed method demonstrates minimal variance, coupled with a relatively high average velocity.
In summary, our proposed cooperative optimization algorithm outperforms other methods in terms of scalability, safety, and traveling efficiency. It showcases robustness against an increase in the planning horizon and the number of CAVs. 

Note that when the number of CAVs involved in one OCP is more than 20, the cooperative optimization problem consumes more than $1.0$~s, which is not sufficient for real-time urban cooperative driving tasks. 
Hence, our proposed graph evolution algorithm is tailored for the context of large-scale cooperative motion planning, which we thoroughly examine in the subsequent part.

\subsection{Large-Scale Cooperative Motion Planning}
\begin{figure*}[t]
\centering
\subfigure[Initial trajectories]{
\includegraphics[width=0.47\linewidth]{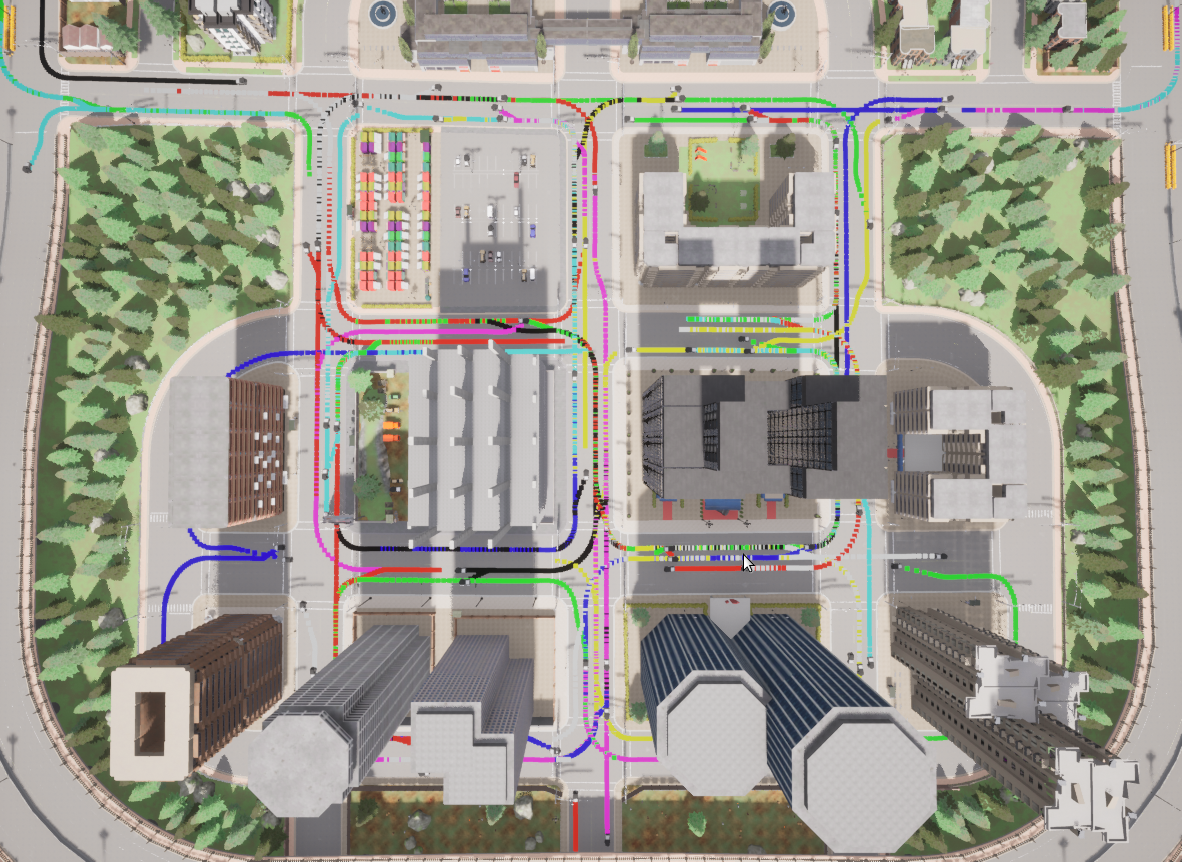} \label{subfig:all_init_trajs_80}
}
\subfigure[Cooperatively planned trajectories]{
\includegraphics[width=0.47\linewidth]{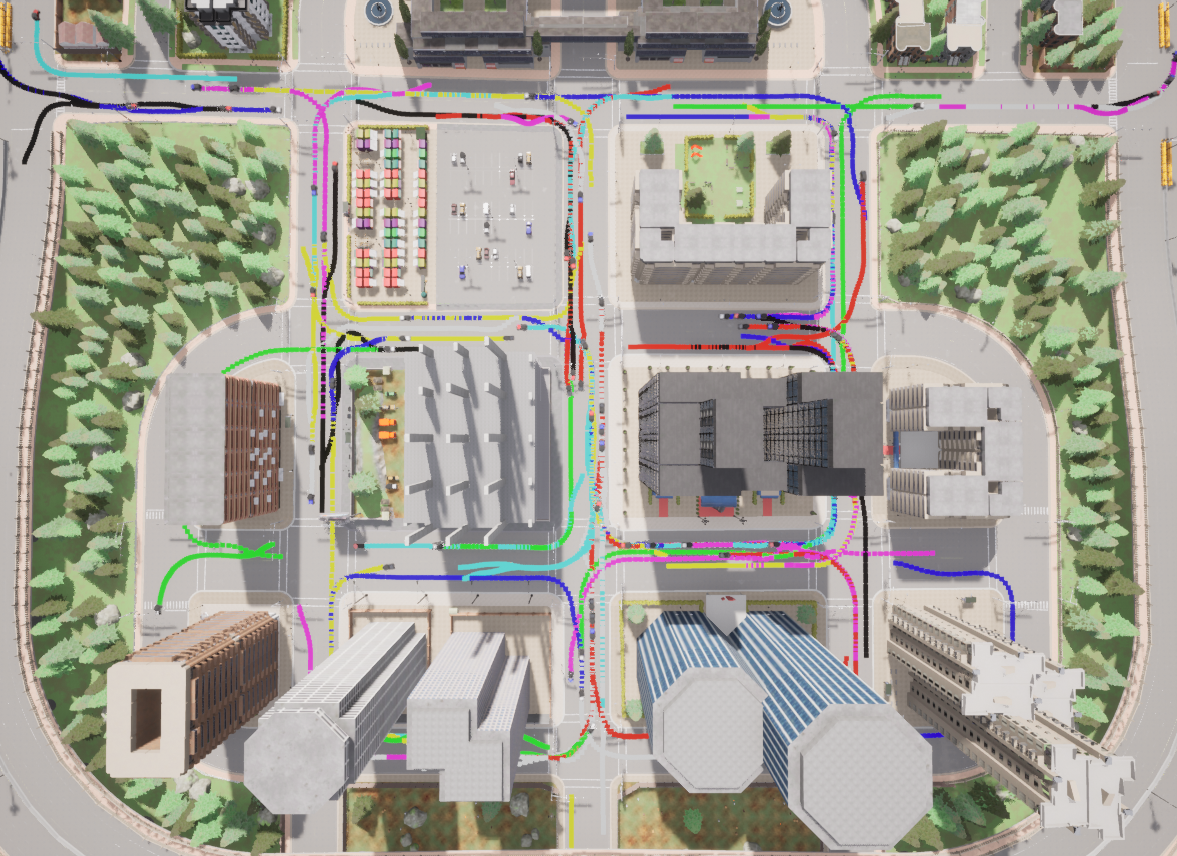} \label{subfig:all_opt_trajs_80}
}
\caption{Rough guidance trajectories and cooperatively planned trajectories for 80 CAVs in the \texttt{Town 05}. Different colors represent the trajectories for different CAVs. Note that there are many potential collision regions in the rough guidance trajectories.}
\label{fig:all_trajs_80}
\end{figure*}
\begin{figure}[t]
    \centering
    \includegraphics[width=1\linewidth]{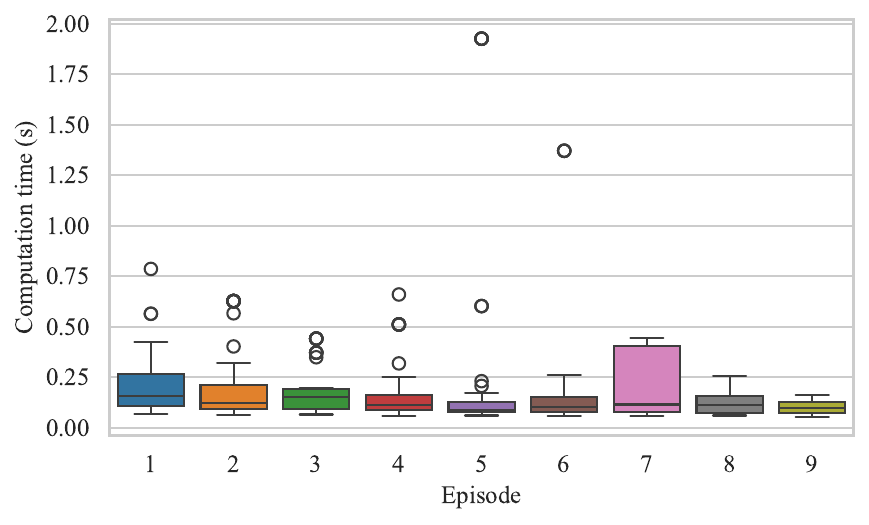}
    \caption{Computation time for each CAV in all the subgraphs. Most of the problems can be solved within 0.25 seconds, and the computation time has a small standard deviation.}
    \label{fig:NodeDistribComputingTime}
\end{figure}
\begin{figure}[t]
    \centering
    \includegraphics[width=1\linewidth]{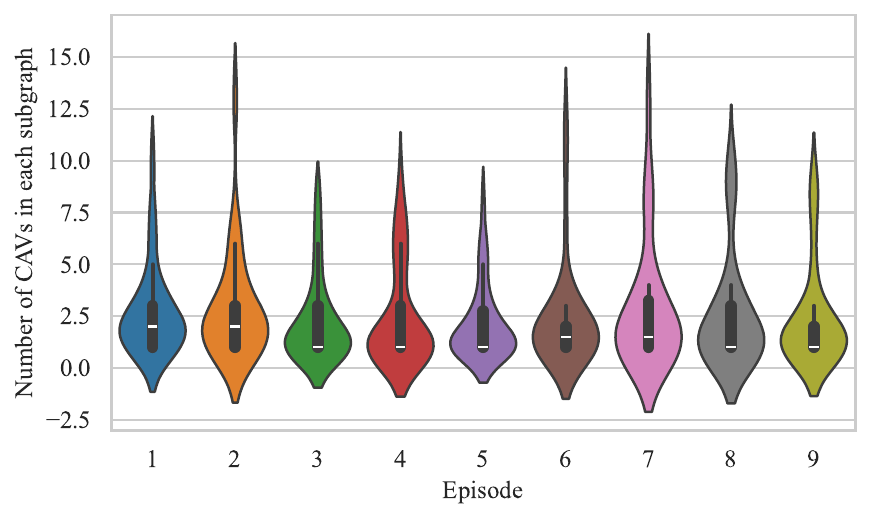}
    \caption{The number of CAVs in each subgraph within different episodes in the proposed cooperative motion planning framework.}
    \label{fig:cliqueDistribCavNumber}
\end{figure}
In this section, we implement closed-loop cooperative motion planning framework for large-scale CAVs. The objectives of this experiment are to evaluate the robustness, scalability, and effectiveness of the proposed cooperative motion planning framework including the graph evolution and receding horizon strategies. We analyze node distribution situations and qualitative rationality of the CAVs' behavior. To achieve this, we construct a randomly generated large-scale urban driving scenario containing 80 CAVs with varied target velocities as depicted in Table~\ref{tab:paramsForCavGeneration}. The guidance trajectories of the CAVs in this scenario are shown in Fig.~\ref{subfig:all_init_trajs_80}. This random scenario generation setting enables us to simulate complex urban driving tasks and validate the framework's robustness. We use a receding horizon strategy with a control horizon of $T_e=10$ and planning horizon of $T_s = 15$.
\begin{figure}[t]
\centering
\subfigure[Node sets distribution at $\tau = 20$]{
\includegraphics[width=0.46\linewidth]{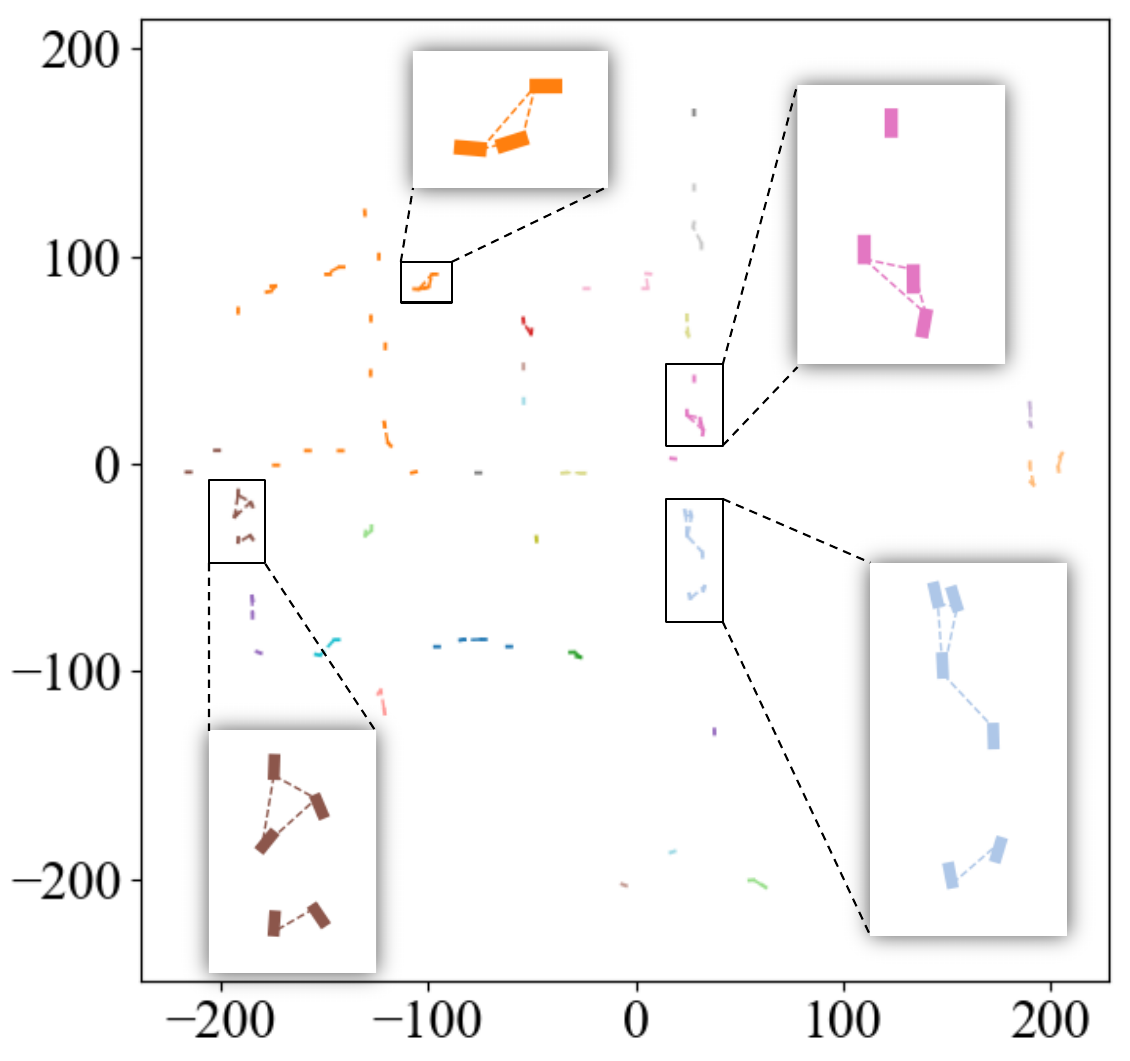} \label{subfig:clique_d_2s}
}
\subfigure[Node sets distribution at $\tau = 60$]{
\includegraphics[width=0.46\linewidth]{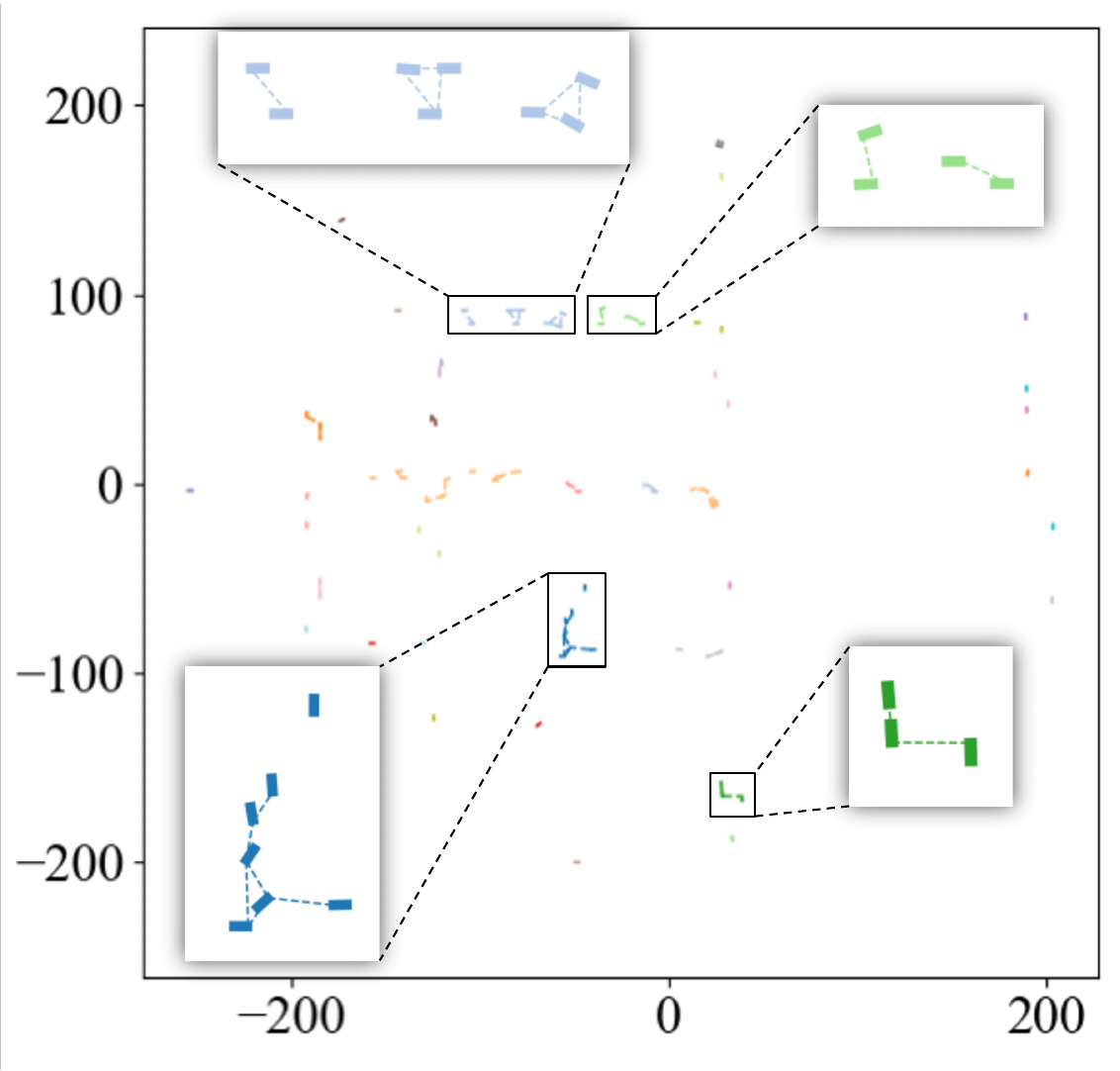} \label{subfig:clique_d_6s}
}
\caption{Node sets distribution at different time steps. Different color means different node sets of CAVs, and the dashed lines between each pair of vehicles mean the communication ability with each other.}
\label{fig:NodeDistribVisual}
\end{figure}
\subsubsection{Effectiveness Analysis of the Node Distribution}
Fig.~\ref{fig:NodeDistribComputingTime} depicts the computation time in each episode for the planning horizon within each subgraph $\mathcal{H}$ generated from the dynamic connectivity graph $\mathcal{G}$ of the 80 CAVs. In most cases, the computation time is less than 0.25~s with a small standard variance, which indicates that the computation time of the algorithm does not fluctuate much under various urban driving conditions and real-time planning can be achieved.

Note that such stable and efficient computing time is inseparable from our proposed graph evolution algorithm. With this, the subsets of CAVs $\mathcal{N}$ are redistributed and updated with an appropriate size through the whole cooperative driving process.
As demonstrated in Fig.~\ref{fig:cliqueDistribCavNumber} and Fig.~\ref{fig:NodeDistribVisual}, the sizes of the subgraphs, determined using the proposed conditional Manhattan distance, are relatively small with less than 10 CAVs. Note that in moments with high CAV density, such as at intersections or long queues on a straight road, the scale of the subgraphs increases slightly, such as the sugraph distribution of CAVs at $\tau=60$ in Fig.~\ref{subfig:clique_d_6s}, which is the main reason for the several high computation time recorded in Fig.~\ref{fig:NodeDistribComputingTime}.

\begin{figure}[t]
\centering
\subfigure[Cooperative Overtaking]{
\includegraphics[width=0.42\linewidth]{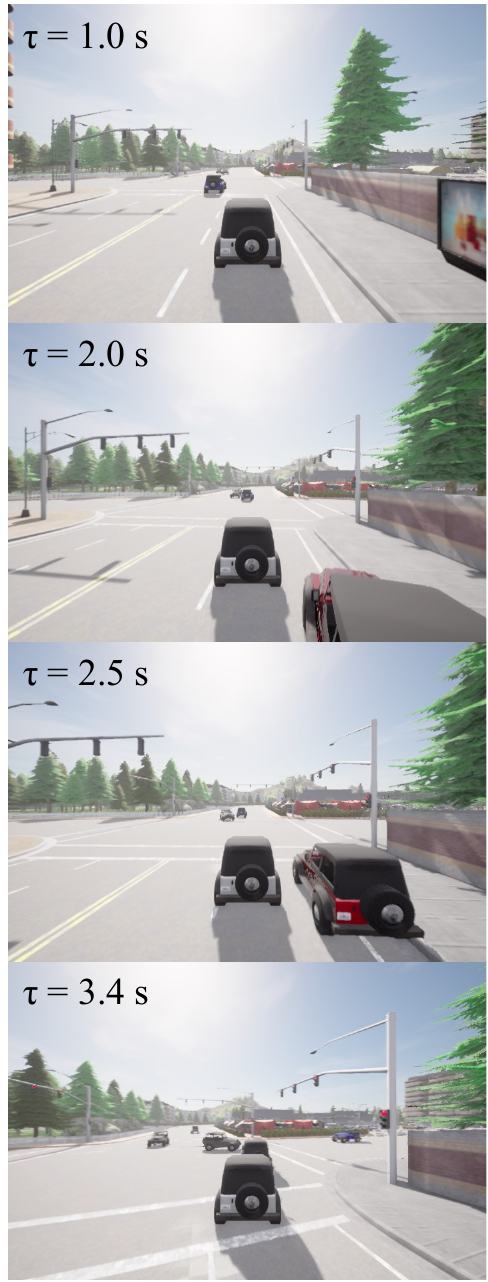} \label{subfig:overtaking}
}
\subfigure[Intersection Crossing]{
\includegraphics[width=0.42\linewidth]{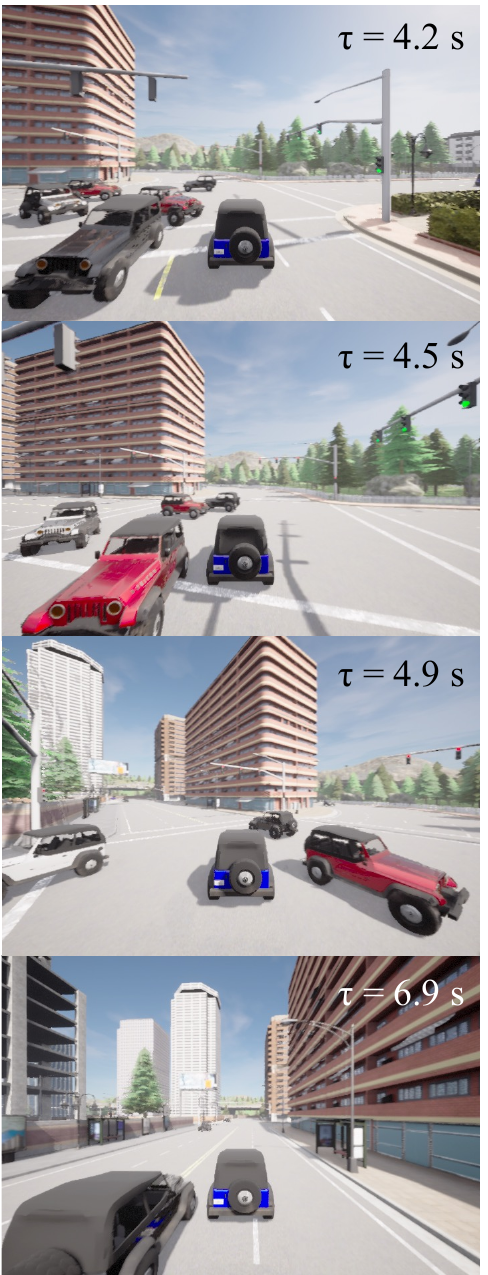} \label{subfig:intersection}
}
\caption{Driving performance in two challenging scenarios simulated in CARLA. In (a), the red CAV overtakes the leading one with minor steering and deviation from its road lane holding a narrow safety distance, due to the slight avoidance of the leading vehicle. In (b), all the CAVs collaboratively drive with a relatively high safety distance without collision in dense conditions.}
\label{fig:demo}
\end{figure}
\subsubsection{Cooperative Driving Performance Evaluation}

In this part, we thoroughly evaluate the cooperative driving performance in the high-fidelity urban traffic scenario, taking into account safety and rationality indices. 
A comparison between the initial trajectories and the cooperatively optimized driving trajectories shown in Fig.~\ref{fig:all_trajs_80} reveals that the optimized trajectories exhibit more curves, even on straight road lanes. This can be attributed to the overtaking and active avoidance behaviors of the CAVs. It is worth noting that these trajectories are sufficiently smooth, allowing the CAVs to adjust their heading angles $\theta^i_\tau$ with slight steering $\varphi^i_\tau$, thereby enhancing the comfort of passengers in CAVs within the traffic system.

For typical and challenging situations, such as overtaking or intersection areas, CAVs exhibit a proactive inclination to evade one another, thereby enhancing both driving effectiveness and safety performance. A notable illustration can be observed in Fig.~\ref{subfig:overtaking}, where the red CAV successfully executes an overtaking maneuver without compromising driving efficiency, due to the active avoidance of the white CAV on the left.
Another noteworthy challenge in urban driving pertains to navigating intersections. As depicted in Fig.~\ref{subfig:intersection}, all the CAVs efficiently drive without stopping and giving way to others in the unsignalized intersection, obviously increasing the traveling efficiency.
The reasons for the above safe and efficient driving performance are the adoption of a cooperative motion planning optimization strategy, aligned with the driving strategy outlined in Algorithm \ref{alg:alg3} and the utilization of an appropriate CAV scale within each subgraph.

\section{Conclusion}

In this work, a comprehensive cooperative motion planning framework is devised to address the cooperative motion planning problem in large-scale traffic scenarios involving excessive numbers of CAVs.  
This framework integrates trajectory generation, graph evolution, and parallel optimization for cooperative urban driving. We enhance computational efficiency of the dual consensus ADMM, by leveraging the sparsity in the dual update of box constraints. The introduction of a graph evolution strategy with the receding horizon effectively manages the scale of each OCP in each subgraph $\mathcal{H}$, regardless of the number of CAVs in graph $\mathcal{G}$.
The simulation results of cooperative motion planning within one subgraph showcase the superiority of the proposed improved consensus ADMM in terms of computational efficiency and driving efficiency, especially with higher planning horizons and increasing numbers of CAVs. We also conduct comprehensive simulations for cooperative motion planning tasks involving 80 CAVs. With the proposed methodology, the computation time for the problem is maintained to be under 0.5\,\textup{s} in most instances. The CAVs exhibit exceptional cooperation behaviors in different driving scenarios. These findings effectively showcase the capabilities of the proposed cooperative driving framework in managing large-scale CAVs, offering valuable insights for efficient coordination in modern transportation networks.

\bibliographystyle{ieeetr}
\bibliography{refs}

\end{document}